\documentclass[letterpaper]{article} 
\usepackage{aaai19}  
\usepackage{times}  
\usepackage{helvet}  
\usepackage{courier}  
\usepackage{url}  
\usepackage{graphicx}  
\usepackage{amsthm}
\usepackage{amsmath}
\usepackage{mathtools}
\newtheorem{theorem}{Theorem}
\newtheorem{lemma}[theorem]{Lemma}

\newtheorem{corollary}[theorem]{Corollary}
\newtheorem{problem}{Problem}

\usepackage[utf8]{inputenc} 
\usepackage[T1]{fontenc}    
\usepackage{hyperref}       
\usepackage{url}            
\usepackage{booktabs}       
\usepackage{amsfonts}       
\usepackage{nicefrac}       
\usepackage{microtype}      
\usepackage{amssymb}
\usepackage{wasysym}
\usepackage{bbm}
\usepackage{mathrsfs}
\usepackage{algorithmicx}{\tiny }
\usepackage{algorithm}
\usepackage[noend]{algpseudocode}
\usepackage{graphicx, caption, subcaption}
\usepackage{natbib}

\nocopyright

\title{PAC Ranking from Pairwise and Listwise Queries: Lower Bounds and Upper Bounds}
\author{Wenbo Ren$^1$, Jia Liu$^2$, Ness Shroff$^1$\\
1: Department of Computer Science, Ohio State University, Columbus, OH 43210\\
2: Department of Computer Science, Iowa State University, Ames, Iowa 50011\\
ren.453@osu.edu, jialiu@iastate.edu, shroff.11@osu.edu}

\begin{document}
	\maketitle
	
	\begin{abstract}
		This paper explores the adaptive (active) PAC (probably approximately correct) top-$k$ ranking (i.e., top-$k$ item selection) and total ranking problems from $l$-wise ($l\geq 2$) comparisons under the multinomial logit (MNL) model. By adaptively choosing sets to query and observing the noisy output of the most favored item of each query, we want to design ranking algorithms that recover the top-$k$ or total ranking using as few queries as possible. For the PAC top-$k$ ranking problem, we derive a lower bound on the sample complexity (aka number of queries), and propose an algorithm that is sample-complexity-optimal up to an $O(\log(k+l)/\log{k})$ factor. When $l=2$ (i.e., pairwise comparisons) or $l=O(poly(k))$, this algorithm matches the lower bound. For the PAC total ranking problem, we derive a tight lower bound, and propose an algorithm that matches the lower bound. When $l=2$, the MNL model reduces to the popular Plackett-Luce (PL) model. In this setting, our results still outperform the state-of-the-art both theoretically and numerically. We also compare our algorithms with the state-of-the-art using synthetic data as well as real-world data to verify the efficiency of our algorithms.
	\end{abstract}
	
	\section{Introduction}\label{sec:Intro}
	The problem of ranking a set of $n$ items is of significant recent interest due to its popularity in various machine learning problems, such as recommendation systems \citep{RecommendationSystem2010}, web search \citep{WebSearch2001}, crowd sourcing \citep{CrowdSourcing2013}, and social choices \citep{SocialChoice2011Lu,SocialChoice2013Caragiannis,SocialChoice2005Conitzer}. The goal of these ranking problems is to recover a total or partial ranking from noisy queries (aka samples) about users' preferences. A query presents a user with a set of $l$ items, such as products, movies, pages, and candidates, and asks him/her to select the most favored item. An interesting area of study is active ranking, where the learner can actively select the items to be queried based on past query results in order to reduce the number of queries needed (sample complexity). The work in \citep{CrowdSourcing2013} shows that adaptive ranking algorithms can achieve almost the same accuracy as non-adaptive ones using only about 3\% of the samples.
	
	In this paper, we focus on listwise ranking (i.e., each time, the learner can query more than two items) instead of querying only two items at a time, i.e., pairwise ranking. There are several motivations to study listwise ranking, which is a relatively unexplored area compared to pairwise ranking. Results on listwise ranking can be directly applied to pairwise ranking, which is a special case of listwise ranking, and numerical results in this paper indicate that traditional algorithms designed for pairwise ranking problems do not work well on listwise settings. More importantly, in many applications such as web search and online shopping, presenting more than two items to the users is more common and typical, and can provide better user experiences. In these applications, using adaptive listwise ranking algorithms, the server can adaptively choose items to present and learn the users' preference in a shorter time.
	
	There are mainly two classes of ranking problems, one aims to find the $k$ most preferred items, and the other aims to recover the total ranking (or full ranking). This paper studies both problems. Instead of exact ranking, this paper focuses on the PAC ranking \citep{OnlineRankingElicitation2015,MaxingAndRanking2017,Falahatgar2017}, in which an $\epsilon$-bounded error on the preference scores do not influence the correctness. See more detailed definition of PAC in Section~\ref{sec:PF}. We also consider the exact ranking problem based on our results on PAC ranking.
	
	\section{Problems Formulation}\label{sec:PF}
	Under the multinomial logit (MNL) model, each item is associated with a preference score represented by a real number. A more preferred item has a larger preference score. The items are ranked based on their preference scores. 
	
	A query over a set $S=\{i_1,i_2,...,i_l\}$ will return $i_m$ as the most favored item with probability
	\begin{equation}
		\mathbb{P}(i_m\mid S) = \frac{\exp(\gamma_{i_m})}{\sum_{j=1}^l\exp(\gamma_{i_j})},
	\end{equation}
	where $\gamma_{i_m}$ is the preference score of item $i_m$. The MNL model was introduced by \citep{Bradley1952}, and has been widely adopted in many areas \citep{MMMethod2004}. We also assume that the queries are independent across time and items. Mathematically speaking, if we query $t$ sets (some of them can be the same) $S_1,S_2,...,S_t$, where $S_\tau=\{i_{\tau,1},i_{\tau,2},...,i_{\tau,l}\}$, then we will get query result $(i_{a_1},i_{a_2},...,i_{a_t})$ with probability $\prod_{\tau=1}^t{\mathbb{P}(i_{a_\tau}|S_\tau)}=\prod_{\tau=1}^t{\frac{\exp(\gamma_{i_{a_\tau}})}{\sum_{b=1}^l{\exp(\gamma_{i_{\tau,b}})}}}$. Items with larger preference scores are more likely to win a query (i.e., be the query result), and thus, the items are ranked through this hidden information.
	
	When $l=2$, the MNL model reduces to the Bradley-Terry-Luce (BTL) model  \citep{Bradley1952}, which is equivalent to the Plackett-Luce (PL) model. Under the PL model, item $i$ will be the result of a query over $i$ and $i$ with probability $\frac{\theta_i}{\theta_i+\theta_j}$, where $\theta_i = \exp(\gamma_i)$ and $\theta_j=\exp(\gamma_j)$.
	
	To simplify the notations, in this paper, we let $\theta_i = \exp(\gamma_i)$, for all items $i$. In this paper, we only use $\theta_i$'s instead of $\gamma_i$'s to avoid ambiguity.
	
	Now, assume that there are a total of $n$ items indexed by $1,2,...,n$, and we use $[n]:=\{1,2,...,n\}$ to denote the set of all items. Since only the ratios of $\theta_i$'s matter, in this paper we normalize $\max_{i\in[n]}{\theta_i} = 1$.
	Define $\theta_{[i]}$ as the $i$-th largest preference score of all items. For $\epsilon\in(0,1)$, an item is said to be $(\epsilon, k)$-optimal if its preference score is no less than $\theta_{[k]}-\epsilon$. For $\epsilon\in(0,1)$, we define 
	\begin{equation}\label{Ukepsilon}
		U_{k,\epsilon} := \{i\in[n]:\theta_i\geq \theta_{[k]}-\epsilon \},
	\end{equation} i.e.,
	$U_{k,\epsilon}$ is the set of all $(\epsilon, k)$-optimal items. A set $R$ is said to be $\epsilon$-top-$k$ if $|R|=k$ and $R\subset U_{k,\epsilon}$, i.e., all items in it are $(\epsilon,k)$-optimal. Here we given a simple example. Let $[n]=\{1,2,3,4\}$, and the preference scores are 1.0, 0.9, 0.89, 0.87 respectively. Let $k=2$ and $\epsilon=0.02$. We have $\theta_{[k]} = 0.9$, $U_{k,\epsilon} = \{1,2,3\}$, and every 2-sized subset of $U_{k,\epsilon}$ is $\epsilon$-top-$k$. Now we define the PAC top-$k$ item selection problem. 
	
	\begin{problem}\label{problem1}
		[PAC Top-$k$ Item Selection ($k$-IS)] Given a set of $n$ items, $k\in\{1,2,3,...,\left\lfloor n/2 \right\rfloor\}$, $l\in\{2,3,4,...,n\}$, $\delta \in (0,1)$, and $\epsilon\in(0,1)$, we want to find a correct $\epsilon$-top-$k$ subset with error probability no more than $\delta$, and use as few $l$-wise queries as possible.
	\end{problem}
	
	Beyond Problem~\ref{problem1}, we further explore the PAC total ranking problem. A function $\Pi$ is said to be a permutation of $[n]$ if it is a bijection from $[n]$ to $[n]$, where $\Pi(i)=j$ means item $i$ ranks the $j$-th largest. Given $\epsilon\in(0,1)$, a permutation $\Pi$ is said to be an $\epsilon$-ranking if for all $i$ and $j$ in $[n]$, $\Pi(i) < \Pi(j)$ (i.e., $i$ ranks higher than $j$) implies $\theta_i \geq \theta_j - \epsilon$. In other words, an $\epsilon$-ranking is a correct ranking except that incorrect orders among items with preference scores' difference no greater than $\epsilon$ are allowed. Now we define the PAC total ranking problem.
	
	\begin{problem}\label{problem2}
		[PAC Total Ranking (TR)] Given a set of $n$ items, $l\in\{2,3,4,...,n\}$, $\delta \in (0,1)$, and $\epsilon\in(0,1)$, we want to find a correct $\epsilon$-ranking with error probability no more than $\delta$, and use as few $l$-wise queries as possible.
	\end{problem}
	
	In this paper, we add two constraints to our problems. The first one is that we can only perform $l$-wise queries. As has been explained in Section~\ref{sec:Intro}, this constraint is reasonable and of interest. The second constraint is that the ratios of preference scores between items are upper bounded by a constant. In this paper, it is referred to as the RBC (ratios bounded by a constant) condition. The RBC condition implies that there exists some constant $C$ such that $\sup_{i,j\in[n]}{\theta_i}/{\theta_j} \leq C$. Under the RBC condition, the least preferred item has a lower bounded probability to win the most preferred item. The RBC condition has been adopted by many previous works \citep{SpectralMLE2015,RankCentrarity2016,LimitedRounds2017,ListwisePL2017,BothOptimal2017}. The rationale behind the RBC condition is as follows.
	
	First, the RBC condition is a good model of the situations where noises are not insignificant, and the least preferred items still have a chance to win the most preferred ones. Second, \citet{Chen2018} showed that if the RBC condition does not hold, one can get an $l$-reduction for Problem~\ref{problem1}. However, when the RBC condition holds, they showed that the sample complexity is lower bounded by $\Omega(n)$, and their algorithms' sample complexity is $O(n\log^{14}{n})$ under default parameters, far higher than the lower bound. Thus, we are interested whether their results can be improved when the RBC condition holds. 	
	
	\section{Related Work}\label{sec:RW}
	To the best of our knowledge, the first and most recent paper that focuses on listwise active ranking under the MNL model was \citep{Chen2018}, which proposed an algorithm that finds the top-$k$ out of $n$ items with high probability using $O(n \log^{14}{n})$\footnote{All $\log$ in this paper, unless explicitly noted, are natural log.} $l$-wise comparisons when using default parameters, and can obtain up to an $l$-reduction if the preference scores vary significantly. However, when the RBC condition holds, \citet{Chen2018} showed that $\Omega(n)$ queries are necessary, and their algorithms suffer from a $\log^{14}{n}$ factor, which can be large. Motivated by their work, we investigate whether we can tighten the lower bound or remove the $\log^{14}{n}$ factor when the RBC conditions hold.
	
	When $l=2$, the MNL model reduces to PL model \citep{BTLModel2012individual}. Under this model, active ranking has been studied extensively. For the top-$k$ ranking problem, to our knowledge, the best asymptotic result was given by \citep{LimitedRounds2017}. Given $\Delta_k$, the minimal difference between the preference scores of the $k$-th preferred item and the others, its top-$k$ ranking algorithm returns a correct solution with error probability no greater than $\delta$ using $O(\Delta_k^{-2} n \log{(k/\delta)})$ comparisons, which meets the lower and upper bound proved in this paper. However, a weakness of the algorithm in \citep{LimitedRounds2017} is that one needs to know the $\Delta_k$ value a priori. \citet{LimitedRounds2017} also proved an $\Omega(\Delta_k^{-2} n \log{(1/\delta)})$ lower bound on sample complexity. Further, they showed that any algorithm able to solve the (PAC) full exploration multi-armed bandit (FEMAB) problem can solve the (PAC) top-$k$ ranking problem by transforming the latter to the former. Thus, FEMAB algorithms such as that in \citep{Halving2010,LowerBound2012,Top-kBernoulli2015,PureExploration2016} also meet the lower bound proved in this paper. The algorithm in \citep{LimitedRounds2017} also relies on this kind of transformation. However, numerical results in this paper shows that such a transformation performs poorly when $l$ is large. 
	
	For the pairwise total ranking problem, the best theoretical result so far has been given by \citet{OnlineRankingElicitation2015} to our knowledge, where they proposed a total ranking algorithm PLPAC-AMPR using $O\left(\epsilon^{-2} n \log{n} \log(n \delta^{-1}\epsilon^{-1})\right)$ comparisons. This result is looser than our upper and lower bound by a $\log{n}$ factor. Though not stated, it can be proved that Borda Ranking in \citep{MaxingAndRanking2017} solves the pairwise total ranking problem with sample complexity $O\left(\epsilon^{-2} n \log(n \delta^{-1})\right)$. However, numerical results in this paper shows that it does not work well in the listwise settings, especially when $l$ is large.
	
	Works in \citep{SpectralMLE2015,RankCentrarity2016,BothOptimal2017} proposed non-adaptive top-$k$ ranking algorithms by using the special property of the PL model, and the \citet{LimitedRounds2017,BothOptimal2017} showed that these algorithms are sample complexity optimal in the non-adaptive setting (i.e., with sample complexity $O(n\log{n})$). \citet{MultiwiseSpectral2017} studied the listwise top-$k$ ranking, but under a different model. Works in \citep{NoisyComputing1994,Ailon2012active,Top-kSelection2013,MaxingAndRanking2017,Falahatgar2017,RankingLimits2018,mohajer2016active} explored maxing and ranking under different settings, and proposed optimal or nearly optimal algorithms. For the Borda-score model, \citet{ActiveRanking2016,CoarseRanking2018} proposed partition (or coarse ranking) algorithms that solve the top-$k$ item selection problem with sample complexity $O(n\log{n})$, and \citet{MaxingAndRanking2017} explored the maxing and total ranking.
	
	
	\section{Lower Bound Analysis}\label{sec:LBA}
	In this section, we establish the sample complexity (number of queries needed) lower bounds for the two problems defined above. Both the lower bounds in this paper are for the worst case. We do not consider average lower bounds in this paper, since they necessitate assumptions on a prior distribution on the preference scores, which is beyond the scope of this paper. For instance, when deriving the $\Omega(n\log{n})$ lower bound of sorting, people normally assume that all numbers are distinct and each permutation has the same prior probability. There are instances like the one where all numbers take value in $\{0,1\}$, for which $O(n)$ time is sufficient for sorting. Also, due to the PAC setting, there are instances (e.g., preference scores of the items are closer than $\epsilon$) whose ranking can be recovered even by a constant number of queries.
	
	We note that the lower bounds derived in this paper are not restricted to the problems defined in this paper, and can also be applied to others \citep{OnlineRankingElicitation2015,ActiveRanking2016,LimitedRounds2017,Chen2018}. We will provide more detailed discussions after presenting the lower bounds for the PAC problems.

\subsection{Lower Bound for the $k$-IS Problem}
	First, we establish the worst case lower bound for the $k$-IS problem (Problem~\ref{problem1}) stated in Theorem~\ref{LB-k-IS}. 
	
	\begin{theorem}[Lower bound of the $k$-IS problem]\label{LB-k-IS}
		Given $\epsilon \in (0, \sqrt{1/32}]$, $\delta \in (0, 1/4)$, $6 \leq k \leq n/2$, and  $2\leq l \leq n$, there is an instance such that to find an $\epsilon$-top-$k$ subset with error probability no more than $\delta$, any algorithm must conduct $\Omega( \frac{n}{\epsilon^2}\log{\frac{k}{\delta}} )$ $l$-wise queries in expectation.
	\end{theorem}
	\begin{proof}
		We prove that any algorithm able to solve the $k$-IS problem can be transformed to solve the PAC top-$k$ arm selection ($k$-AS) problem with Bernoulli rewards defined in \citep{LowerBound2012}. The sample complexity lower bound of the latter is $\Omega(\frac{n}{\epsilon^2}\log\frac{k}{\delta})$, and completes the proof. See Section~\ref{section9} for details.
	\end{proof}
 	The reduction mentioned above also implies that we can use the ranking algorithms to solve the corresponding FEMAB algorithms, and builds a bridge between ranking and FEMAB problems.
	
	We note that this bound can also be applied to the exact top-$k$ subset selection problem \citep{SpectralMLE2015,RankCentrarity2016,LimitedRounds2017,ListwisePL2017,BothOptimal2017}. Let $\Delta_k:=\theta_{[k]}-\theta_{[k+1]}$. When $\epsilon < \Delta_k$, the unique $\epsilon$-top-$k$ subset is exactly the top-$k$ subset. Thus, for the exact top-$k$ item selection problem, the worst case lower bound is $\Omega(\frac{n}{\Delta_k^2}\log\frac{k}{\delta})$. This bound is higher than the $\Omega(\frac{n}{\Delta_k^2}\log\frac{1}{\delta})$ one derived by \citet{LimitedRounds2017}.
	
	\begin{corollary}[Lower bound of identifying the exact top-$k$ items]\label{LB-k-IS2}
		Define $\Delta_k:=\theta_{[k]}-\theta_{[k+1]}$. Given $\Delta_k \in (0, \sqrt{1/32}]$, $\delta \in (0,1/4)$, $6 \leq k \leq n/2$, and  $2\leq l \leq n$, there is an instance such that to find the exact top-$k$ subset with error probability no more than $\delta$, any algorithm must conduct $\Omega( \frac{n}{\Delta_k^2}\log{\frac{k}{\delta}} )$ $l$-wise queries in expectation.
	\end{corollary}	
	
	To the best of our knowledge, Theorem~\ref{LB-k-IS} is the first known $\Omega( \frac{n}{\epsilon^2}\log{\frac{k}{\delta}} )$ lower bound for the PAC top-$k$ ranking under the MNL model and the PL model, and Corollary~\ref{LB-k-IS2} is the first known $\Omega(\frac{n}{\Delta_k^2}\log{\frac{k}{\delta}})$ lower bound for the exact top-$k$ ranking. Later, in Theorem~\ref{TP-TopK2}, we will show that our lower bounds are tight if $l=2$ (i.e. the pairwise case) or $l=O(poly(k))$. It remains an open problem whether these lower bound are tight for $l>2$.
	
\subsection{Lower Bound for the TR Problem}
	Next, we establish the worst case lower bound for the TR problem (Problem~\ref{problem2}). Recall from the definition of $\epsilon$-ranking that the $k$ highest ranked items in an $\epsilon$-ranking form an $\epsilon$-top-$k$ subset. Thus, the lower bound of the TR problem is no lower than the PAC top-$(n/2)$ selection, i.e., $\Omega(\frac{n}{\epsilon^2}\log\frac{n}{\delta})$. The result is presented in Theorem~\ref{LB-TR}. Later in Theorem~\ref{TP-TR}, we will show that this lower bound is tight. 
	\begin{theorem}[Lower bound for the total ranking problem]\label{LB-TR}
		Given $\epsilon \in (0, \sqrt{1/32}]$, $\delta \in (0, 1/4)$, and $2\leq l\leq n$, there is an instance such that to find a correct $\epsilon$-ranking with error probability no more than $\delta$, any algorithm must conduct $\Omega\left( \frac{n}{\epsilon^2}\log{\frac{n}{\delta}} \right)$
		$l$-wise queries in expectation.
	\end{theorem}
	\begin{proof}
		If one finds an $\epsilon$-ranking of $A$, then the top $\lfloor\frac{n}{2}\rfloor$ items form an $\epsilon$-top-$\lfloor\frac{n}{2}\rfloor$ subset of $A$. The lower bound of the latter one is $\Omega\left( \frac{n}{\epsilon^2}\log{\frac{n}{\delta}} \right)$, and thus, the desired result follows.
	\end{proof}
	
	For the exact total ranking problem, when $n$ increases, the minimal gap of preference scores between the items is of the order $O(1/n)$. To distinguish two items whose preference scores' difference is $O(1/n)$ with probability $3/4$, at least $\Omega(n^2)$ queries are required by Corollary~\ref{LB-k-IS2}. Thus, for any $n$-sized instance, the exact total ranking takes at least $\Omega(n^2)$ queries. The worst case lower bound is $\Omega(n^3\log{n})$ (consider the instance where preferences scores' differences between consecutive items are all $O(1/n)$. 
	
	The lower bounds of these two problems are not dependent on the value of $l$. It indicates that by listwise queries, one can only get up to constant reductions on the sample complexity. However, in practice, numerical results provided in Figure~\ref{fig:lComparison} suggest that when $l$ increases, the number of queries decreases. Furthermore, as noted in Section~\ref{sec:Intro}, for many applications such as web searching and online shopping, listwise queries are more typical and common. When users use these applications, the server, by adaptively presenting items in a listwise manner, can learn the users' preference in a shorter time compared with randomly presenting.
	
\section{Algorithms for the PAC Total Ranking}\label{sec:TR}
	In this section, we present our algorithm called PairwiseDefeatingTotalRanking (PDTR) for Problem~\ref{problem2} (Algorithm~\ref{AL-TR}) and its analysis. Its theoretical performance is stated in Theorem~\ref{TP-TR}. The key idea of this algorithm is to first bound the probability that $j$ wins a query given $i$ or $j$ wins the query (see Lemma~\ref{pi_ij}), and then use this bound to establish a UCB (upper confidence bound)-like method that bounds the probability that an unwanted item is added to the result. The key difference between our algorithm and the UCB-like algorithms for the FEMAB problems (see \citep{LowerBound2012} as an example) is that they bound the empirical means of the bandit arms' rewards, while in our algorithm, we bound the ratio of winning numbers (i.e., $w_j/(w_i+w_j)$) for each pair of items. Our contribution lies in extending the upper confidence bounds of the arm's empirical mean rewards to that of the ratios between items' wins.
		
	\begin{lemma}\label{pi_ij}
		In Algorithms~\ref{AL-TR} with input $\alpha \leq \frac{l-1}{4(l+C-1)}$, for any $i,j$ in $R$ with $\theta_i>\theta_j$, the probability that $i$ wins a query given $i$ or $j$ wins the query is at least $\frac{1}{2}+\alpha(\theta_i-\theta_j)$.
	\end{lemma}
	\begin{proof} We will show that for each set $S$ containing $i$, there is an one-to-one corresponding set $S'$ such that $\frac{Pr\{i\mbox{ wins the query over } S\}}{Pr\{i\mbox{ wins the query over } S\}+Pr\{j\mbox{ wins the query over } S'\}} \geq \frac{1}{2}+\alpha(\theta_i-\theta_j)$, and derive the desired result. See Section~\ref{section10} for details. \end{proof}
	
	\begin{algorithm}[bht]
		\caption{PairwiseDefeatingTotalRanking$(A, \delta, \epsilon, \alpha)$}\label{AL-TR}
		\hspace*{\algorithmicindent} \textbf{Input:} $A$ the $n$-sized set to be ranked, $\delta$ a desired error probability bound, $\epsilon$ the error tolerance, and $\alpha$ a parameter balancing correct probability and sample complexity.\\
		\hspace*{\algorithmicindent} \textbf{Output:} An $\epsilon$-ranking that is correct w.p. $\geq 1-\delta$.\\
		\hspace*{\algorithmicindent} \textbf{Initialize:} 
		$\delta^*_1 \gets \frac{\delta}{n(n-1)+1}$; $lo\gets 1$; $hi\gets n$; $R\gets A$; $\forall i\in A$, $w_i\gets 0$; $\Pi\gets$ empty map; \Comment{$lo$ and $hi$ are pointers; $R$ stores the remaining items; $w_i$ records the wins of item $i$;}
		\begin{algorithmic}[1]
			\Repeat 
			\If{$|R| \geq l$} $S \gets$ a random $l$-sized subset of $R$;
			\Else \ \ $S \gets R\ \cup$\ \{last $l-|R|$ items removed from $R$\};
			\EndIf
			\State Query $S$ once; Let $q$ denote the winner;
			\State $w_q \gets w_q+1$;
			\If{$w_q \geq \frac{1}{4 \alpha^2 \epsilon^2}\log{\frac{1}{\delta^*_1}}$} 
			\State $\forall j\in R-\{q\}$, mark "$q$ \textsl{defeats} $j$";
			\EndIf
			\For{$j \in R$ such that $j$ does not \textsl{defeat} $q$} 
			\If{$\frac{w_q}{w_q+w_j} \geq b_{w_q+w_j}$} 
			mark "$q$ \textsl{defeats} $j$"; \indent \Comment{Def $b_{w_q\!+\!w_j}\!:=\!\frac{1}{2}\!-\!\alpha\epsilon\!+\!\sqrt{\frac{1}{2(w_q+w_j)}\!\log\!{\frac{\pi^2 (w_q\!+\!w_j)^2}{6 \delta^*_1}}}$}
			\EndIf
			\EndFor
			\If{$q$ \textsl{defeats} every other element of $R$} 
			\State $\Pi(q)\gets lo$; $R\gets R - \{q\}$; $lo\gets lo+1$; 
			\EndIf
			\For{$i\in R$}
			\If{$i$ is \textsl{defeated} by every other element of $R$} 
			\State $\Pi(i)\gets hi$; $R\gets R - \{i\}$; $hi\gets hi-1$; 
			\EndIf 
			\EndFor
			\Until{$lo \geq hi$}\\
			\Return $\Pi$
		\end{algorithmic}
	\end{algorithm}
	
	Here we explain the main idea of PDTR. In this algorithm, upper confidence bounds on $\frac{w_j}{w_i+w_j}$ are established to make sure the following event happens with probability at least $1-n(n-1)\delta^*_1$: for all $i,j$ with $\theta_i > \theta_j+\epsilon$, during the time they are in $R$, (i) it always holds that $\frac{w_j}{w_i+w_j} < b_{w_i+w_j}$, and (ii) $w_j$ does not reach $\frac{1}{4\alpha^2\epsilon^2}\log\frac{1}{\delta^*_1}$ before $w_i$. It can be seen from the algorithm that if the above event happens, for all $i,j\in A$ with $\theta_i > \theta_j+\epsilon$, $i$ ranks higher than $j$ in $\Pi$, and thus, the returned value $\Pi$ is correct. We say a query is \textsl{useful} if its query result (i.e., the winner) is in $R$ at the time when the result is revealed, and is \textsl{useless} otherwise. By removing an item as soon as it wins $\frac{1}{4\alpha^2\epsilon^2}\log\frac{1}{\delta^*_1}$ queries, we can bound the number of \textsl{useful} queries by $O(\frac{n}{\alpha^2\epsilon^2}\log\frac{1}{\delta^*_1})$. The number of \textsl{useless} queries is upper bounded by $O(\frac{l}{\alpha^2\epsilon^2}\log\frac{1}{\delta^*_1})$ with probability $1-\delta^*_1$. Thus, the sample complexity is $O(\frac{n}{\alpha^2\epsilon^2}\log\frac{1}{\delta^*_1})$. Based on this intuition, we characterize the theoretical performance of PDTR in Theorem~\ref{TP-TR}. See Section~\ref{section11} for complete proof.
	
	\begin{theorem}[Theoretical performance of PDTR]\label{TP-TR}
		With probability at least $1-\delta$, PDTR terminates after $O(\frac{n}{\alpha^2\epsilon^2}\log\frac{n}{\delta})$ $l$-wise queries, and, if $\alpha \leq \frac{l-1}{4(l+C-1)}$, returns a correct $\epsilon$-ranking.
	\end{theorem}
	
	As $\frac{l-1}{4(l+C-1)} = \Omega(1)$, we can let $\alpha=\Omega(1)$, and PDTR's sample complexity upper bound is $O(\frac{n}{\epsilon^2}\log\frac{n}{\delta})$.
	
	Here, we add an $\alpha$ parameter to the input in order to balance the trade-off between error probability and sample complexity in practice. We note that $\alpha \leq \frac{l-1}{4(l+C-1)}$ is a sufficient condition for PDTR to achieve an error probability no greater than $\delta$ for any input instance. However, in practice, most cases do not need a small $\alpha$ value to achieve the target success probability. So do Algorithms~\ref{AL-TopK1} and ~\ref{AL-TopK2}. 
 
	By Theorem~\ref{LB-TR}, we can see that PDTR is sample complexity optimal in order sense. We note that \citet{MaxingAndRanking2017} proposed algorithms with the same upper bound for the special case $l=2$. However, when $l>2$, they have no theoretical guarantees, and numerical results in Section~\ref{sec:NR} indicate that our algorithm outperforms theirs. 
	
\section{Algorithms for the PAC Top-$\mathbf{k}$ Item Selection}\label{sec:kIS}
	
	In this section, we provide the algorithm TournamentKSelection (TNKS) for Problem~\ref{problem1} (Algorithm~\ref{AL-TopK2}). This algorithm is inspired by "Halving" proposed in \citep{Halving2010}. "Halving" divides $\delta$ and $\epsilon$ into $\delta_r$'s and $\epsilon_r$'s, and eliminate half the remaining items for each round while guaranteeing $(\delta_r,\epsilon_r)$-correctness of this round. We first modify PDTR to establish the PairwiseDefeatingKSelection (PDKS) algorithm (Algorithm~\ref{AL-TopK1}). PDKS has a special property (stated in Lemma~\ref{kappaCorrectness}) that will be used in the establishment of TNKS. Then, we use the similar ideas as Halving to design TNKS, which solves the $k$-IS problem with $O(\frac{n}{\epsilon^2}\log\frac{k+l}{n})$ sample complexity.
	
	We first present PDKS, which is similar to PDTR with the difference being that the former returns immediately after $k$ $(\epsilon,k)$-optimal items are found. The sample complexity of PDKS is still $O(\frac{n}{\epsilon^2}\log\frac{n}{\delta})$, but the constant factor is smaller as it can be viewed as an early-stopped version of PDTR. 
	
	\begin{algorithm}[bht]
		\caption{PairwiseDefeatingKSelection$(A, k, \delta,\epsilon, \alpha)$}\label{AL-TopK1}
		\hspace*{\algorithmicindent} \textbf{Input:} $A$ the $n$-sized set to be ranked, $k$ the number of top items to be selected, $\delta$ a desired error probability bound, $\epsilon$ the error tolerance, and $\alpha$ a parameter balancing success probability and sample complexity.\\
		\hspace*{\algorithmicindent} \textbf{Initialize:} 
		$\delta^*_2\gets \frac{\delta}{2k(n-1)+1}$;
		$Ans\gets \emptyset$; 
		$R \gets A$;
		$\forall i\in A$, $w_i\gets 0$; \Comment{$w_i$ stores item $i$'s number of wins;}
		\begin{algorithmic}[1]
			\Repeat 
			\If{$|R| \geq l$} $S \gets$ a random $l$-sized subset of $R$;
			\Else \ \ $S \gets R\ \cup$\ \{last $l-|R|$ items removed from $R$\};
			\EndIf
			\State Query $S$ once; Let $q$ be the winner; 
			\State $w_q\gets w_q+1$;
			\If{$w_q \geq \frac{1}{4 \alpha^2 \epsilon^2}\log{\frac{1}{\delta^*_2}}$} 
			\State $\forall j\in R-\{q\}$, mark "$q$ \textsl{defeats} $j$";
			\EndIf
			\For{$j \in R$ such that $j$ does not \textsl{defeat} $q$} 
			\If{$\frac{w_q}{w_q+w_j} \geq b_{w_q+w_j}$} 
			mark "$q$ \textsl{defeats} $j$";
			\indent \Comment{Def $b_{w_q\!+\!w_j}\!:=\!\frac{1}{2}\!-\!\alpha\epsilon\!+\!\sqrt{\frac{1}{2(w_q+w_j)}\!\log\!{\frac{\pi^2 (w_q\!+\!w_j)^2}{6 \delta^*_2}}}$}
			\EndIf
			\EndFor
			\If{$i$ \textsl{defeats} every other element of $R$} 
			\State $Ans\gets Ans\cup \{i\}$; $R\gets R - \{i\}$; 
			\EndIf
			\For{$i\in R$}
			\If{$i$ is \textsl{defeated} by every other element of $R$} 
			\State $R\gets R - \{i\}$;\Comment{Discard $i$}
			\EndIf 
			\EndFor
			\Until{$|Ans|=k$}\\
			\Return $Ans$
		\end{algorithmic}
	\end{algorithm}	
	
	Lemma~\ref{kappaCorrectness} is a property of PDKS, which will be used to establish TNKS. The theoretical performance of PDKS is stated in Theorem~\ref{AL-TopK1}.
	
	\begin{lemma}\label{kappaCorrectness}
		Let $\kappa\in\{1,2,...,k\}$ be arbitrary. With probability at least $1-\frac{2\kappa(n-1)\delta}{2k(n-1)+1}$, the returned value of PDKS contains at least $\kappa$ items whose preference scores are no less than $\theta'_{[\kappa]}-\epsilon$, where $\theta'_{[\kappa]}$ is the $\kappa$-th largest preference score among all items in $A$.
	\end{lemma}
	\begin{proof} The proof is almost the same as that of Theorem~\ref{TP-TR}. See Section~\ref{section12} for details. \end{proof}
	
	\begin{theorem}[Theoretical performance of PDKS]\label{TP-TopK1}
		With probability at least $1-\delta$, PDKS terminates after $O(\frac{n}{\alpha^2\epsilon^2}\log\frac{n}{\delta})$ $l$-wise queries, and, if $\alpha\leq  \frac{l-1}{4(l+C-1)}$, returns a correct $\epsilon$-top-$k$ subset of $A$.
	\end{theorem}
	\begin{proof}
		First fix $\alpha\leq \frac{l-1}{4(l+C-1)}$. Letting $\kappa=k$ in Lemma~\ref{kappaCorrectness}, we have that with probability at least $1-2k(n-1)\delta^*_2$, the returned value is correct. 
		As for the sample complexity, we consider all positive $\alpha$ values. The number of \textsl{useful} queries is obviously upper bounded by $O(\frac{n}{\alpha^2\epsilon^2}\log\frac{1}{\delta^*_2}) = O(\frac{n}{\alpha^2\epsilon^2}\log\frac{n}{\delta})$ (recall that $\delta^*_2=\Theta({\delta}/{n})$). By Chernoff Bound and some computation, we can prove that the number of \textsl{useless} queries is at most $O(\frac{l}{\alpha^2\epsilon^2}\log\frac{n}{\delta})$ with probability $1-\delta^*_2$. The desired result follows. See Section~\ref{section13} for details.
	\end{proof}
	
	Based on PDKS, we design TNKS. TNKS runs like a tournament. At each round $r$ (i.e., the $r$-th repetition of lines 2 to 7), it divides the remaining items $R$ into groups of size $m\geq 2k$. Then, items within each group compete and only $k$ of them survive. After each round, only half (with at most $k$ more) of the remaining items will survive. Precisely speaking,
	\begin{equation}
		\label{T_rSize} |T_r| \leq \left\lceil {|T_{r-1}|}/{m} \right\rceil k \leq \left\lceil {\left\lceil \frac{n}{k} \right\rceil}{2^{-r}} \right\rceil k, 
	\end{equation}
	where $T_r$ is the set of remaining items after round $r$. Rounds will be repeated until only $k$ items remain. Thus, by at most $\left\lceil \log_2{n} \right\rceil$ rounds, TNKS terminates.
	
	\begin{algorithm}[bht]
		\caption{TournamentKSelection$(A,k, \delta,\epsilon, \alpha)$}\label{AL-TopK2}
		\hspace*{\algorithmicindent} \textbf{Input:} $A$ the $n$-sized set to be ranked, $k$ the number of top items to be selected, $\delta$ a desired error probability bound, $\epsilon$ the error tolerance, and $\alpha$ a parameter balancing success probability and sample complexity.\\
		\hspace*{\algorithmicindent} \textbf{Output:} An $\epsilon$-top-$k$ subset correct w.p. $\geq 1-\delta$.\\
		\hspace*{\algorithmicindent} \textbf{Initialize:} 
		$m \gets \min \{n, \max\{2k, k+l-1\}\}$;
		$T_0 \gets A$;
		$r \gets 0$;
		$\delta_r \gets \frac{6 \delta}{r^2 \pi^2}$, and  
		$\epsilon_r \gets \frac{\epsilon}{4}(\frac{4}{5})^r$ for $r\in \mathbbm{Z}^+$;
		\begin{algorithmic}[1]
			\Repeat
			\State  $r\gets r+1$;  $T_r\gets \emptyset$; $R\gets T_{r-1}$; 
			\Repeat 
			\If{$|R|\geq m$} $B$ $\gets$ \{$m$ random items in $R$\};
			\Else \  $B\gets R\ \cup$ \{$(m-|R|)$ random items\};
			\EndIf
			\State $D \gets$ PDKS$(B,k,\delta_r,\epsilon_r, \alpha)$
			\State $T_r \gets T_r \cup D$, $R\gets R-B$
			\Until{$R=\emptyset$}
			\Until{$|T_r|=k$} \\
			\Return $T_r$
		\end{algorithmic}
	\end{algorithm}
	
	By using Lemma~\ref{kappaCorrectness}, we can prove the theoretical performance of TNKS, which is stated in Theorem~\ref{TP-TopK2}. 
	
	\begin{theorem}[Theoretical performance of TNKS]\label{TP-TopK2}
		With probability at least $1-\delta$, TNKS terminates after $O(\frac{n}{\alpha^2\epsilon^2}\log\frac{k+l}{\delta})$ $l$-wise queries, and, if $\alpha\leq \frac{l-1}{4(l+C-1)}$, returns a correct $\epsilon$-top-$k$ subset of $A$.
	\end{theorem}
	\begin{proof} 
		We can prove that when $|T_{s-1}\cap U_{k,\sum_{r=1}^{s-1}\epsilon_r}|\geq k$ ($U_{k,\sum_{r=1}^{s-1}\epsilon_r}$ is defined in (\ref{Ukepsilon})), at round $s$, with probability at least $1-\delta_s$, TNKS takes at most $O(\frac{|T_{s-1}|}{\epsilon^2}\log\frac{m}{\delta_s})$ queries, and $|T_{s}\cap U_{k,\sum_{r=1}^{s}\epsilon_r}|\geq k$. The desired result then follows from the choices of $\delta_r$ and $\epsilon_r$ in TNKS. See Section~\ref{section14} for details. 
	\end{proof}
	
	Fix $\alpha = \Omega(1)$, the sample complexity is $O(\frac{n}{\epsilon^2}\log\frac{k+l}{\delta})$.
	
	Clearly, under the PL model (i.e., $l=2$), our algorithm has order-optimal sample complexity in the worst case. When $l>2$, if $l=O(poly(k))$, our algorithm is still order-optimal in the worst case. When $l$ increases, the theoretical upper bound of TNKS' sample complexity increases. However, it can be seen later from the numerical results in Figure~\ref{fig:lComparison} that as $l$ increases, the actual number of queries decreases. This can be explained as follows: The required $\alpha$ value ($\frac{l-1}{4(C+l-1)}$) decreases as $l$ increases, and the sample complexity upper bound scales as $O(\alpha^{-2})$.
	
\section{Numerical Results}\label{sec:NR}
	In this section, we compare our algorithms with the state-of-the-art by running simulations on synthetic data as well as real-world data. 
	
	\subsection{Synthetic Data}
	In this subsection, we perform comparisons using synthetic data. In the datasets, all items' preference scores are independently generated uniformly at random in $[1/C,1]$, and then rescaled to let the maximal preference score be $1$. All algorithms are tested on the same datasets for fair comparisons. Every point in the figures is averaged over 100 trails. In this part we fix $n=10$, $\epsilon=0.05$, $\delta=0.05$, and $C=10$, and vary the other parameters to perform the comparisons.
	
	We first compare TNKS (Algorithm~\ref{AL-TopK2}) for the top-$k$ ranking problems with the state-of-the-art algorithms including Spectral MLE \citep{SpectralMLE2015}, AlgPairwise (and AlgMultiwise, its listwise version) \citep{Chen2018}, and Halving \citep{Halving2010}. Spectral MLE is with $O(n\log{n})$ sample complexity in the pairwise case. AlgMultiwise's sample complexity is $O(n\log^{14}{n})$ under default parameters. Halving is an example of FEMAB algorithms, with sample complexity $O(\frac{n}{\epsilon^2}\log\frac{k}{\delta})$ for $l=2$.
	
	In the implementations of these algorithms, we vary the parameters (e.g., $\alpha$ of TNKS) to balance the trade-off between success rate and sample complexity. For Spectral MLE, we fix $L$ and vary input parameter $p$. For AlgMultiwise and AlgPairwise, we vary $\kappa$. For Halving, we vary input parameter $\delta$, the desired error probability bound. 
	
	\begin{figure}[bht]
		\begin{subfigure}[b]{0.23\textwidth}
			\includegraphics[scale=0.5]{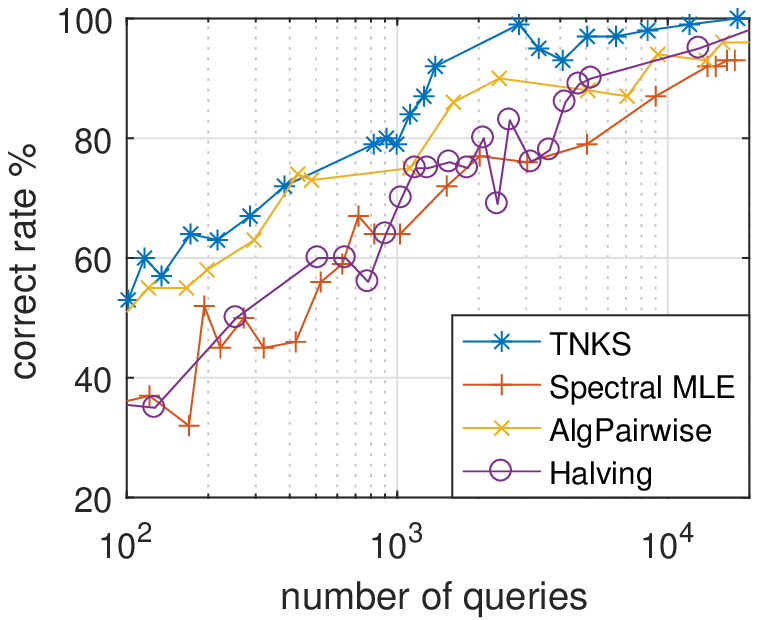}
			\caption{$l=2,k=1$.}
		\end{subfigure}
		\begin{subfigure}[b]{0.23\textwidth}
			\includegraphics[scale=0.5]{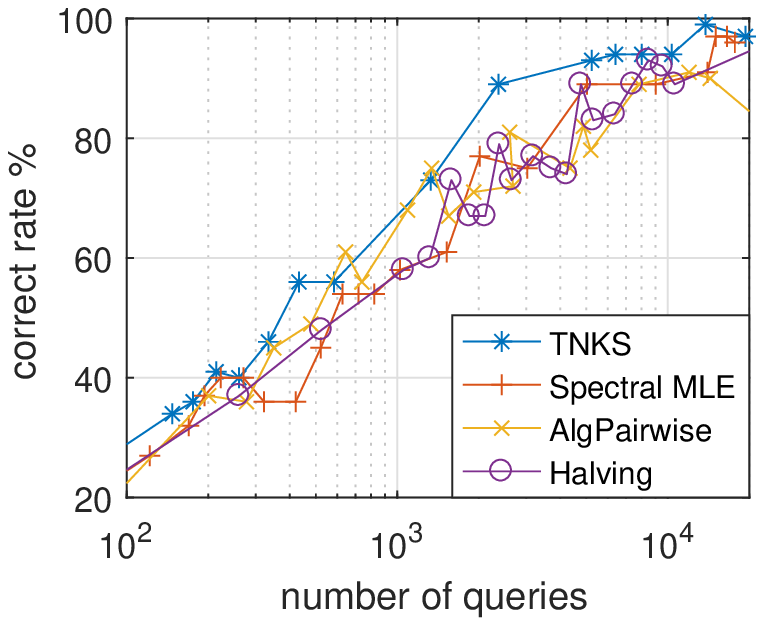}
			\caption{$l=2,k=2$.}
		\end{subfigure}
		\begin{subfigure}[b]{0.23\textwidth}
			\includegraphics[scale=0.5]{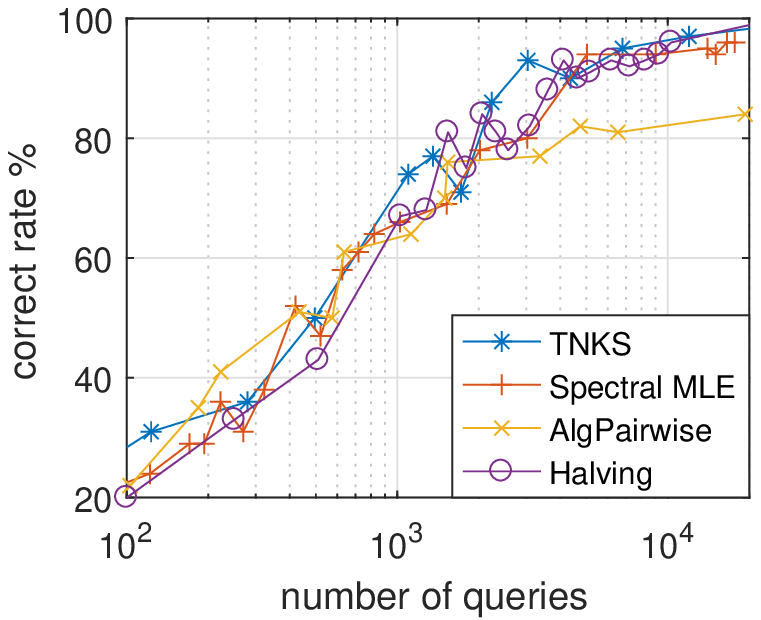}
			\caption{$l=2,k=5$.}
		\end{subfigure}
		\begin{subfigure}[b]{0.23\textwidth}
			\includegraphics[scale=0.5]{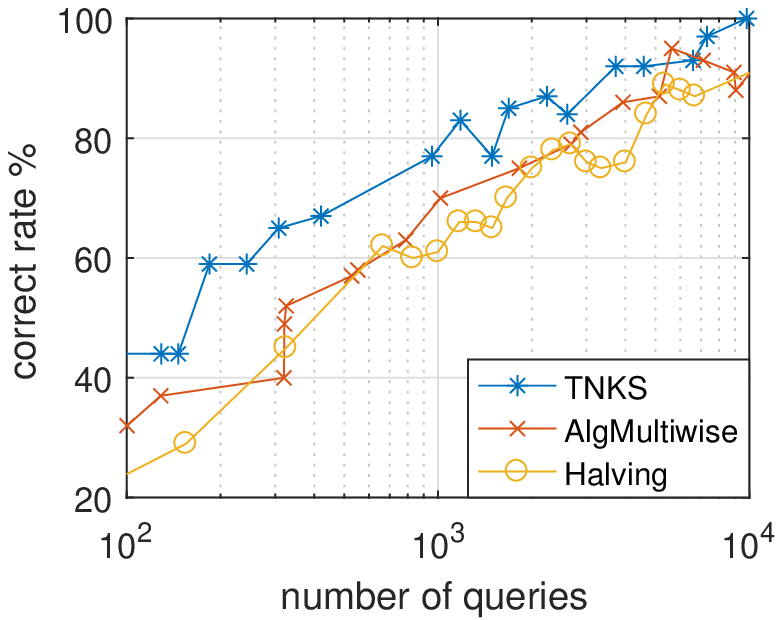}
			\caption{$l=3,k=2$.}
		\end{subfigure}
		\begin{subfigure}[b]{0.23\textwidth}
			\includegraphics[scale=0.5]{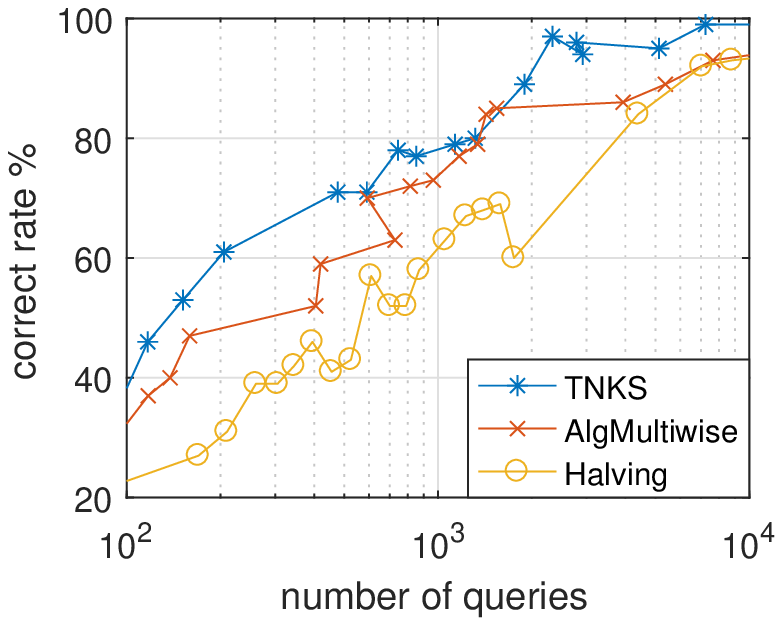}
			\caption{$l=5,k=2$.}
		\end{subfigure}
		\begin{subfigure}[b]{0.23\textwidth}
			\includegraphics[scale=0.5]{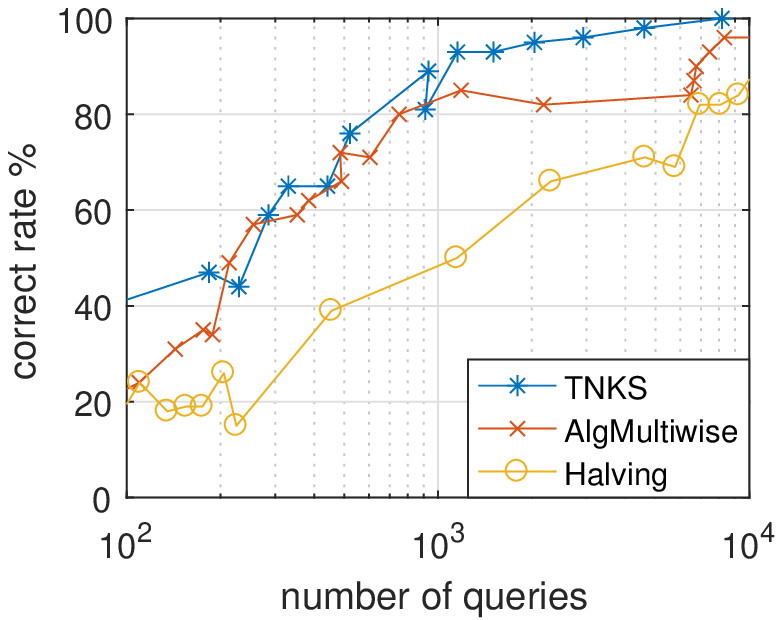}
			\caption{$l=10,k=2$.}
		\end{subfigure}
		\caption{Comparisons of top-$k$ ranking algorithms.}\label{fig:TKSyn}
	\end{figure}

	We begin with the pairwise case (i.e., $l=2$). The results are shown in Figure~\ref{fig:TKSyn} (a)-(c). We can see that when $k=1$ or $k=2$, TNKS outperforms other algorithms. When $k=5$, the performances of all algorithms are close. This is because the sample complexity of TNKS is $O(\frac{n}{\epsilon^2}\log{\frac{k}{\delta}})$, while that of Spectral MLE and AlgPairwise is $O(n \cdot poly(\log{n}))$, so TNKS performs better when $k$ is small. The results indicate that the advantage of TNKS is greater when $k$ is small, consistent with our theoretical results.
	
	Next, we compare these algorithms in the listwise case (i.e, $l>2$). The results are illustrated in Figure~\ref{fig:TKSyn} (d)-(f). Spectral MLE works only for pairwise ranking, which is not comparable in this part. As we can see, TNKS' performance is better than AlgMultiwise overall, consistent with our theoretical results. Also, TNKS is better than Halving. Further, it can be seen that when $l$ increases, the gap between TNKS and Halving increases. This indicates that the approach of transforming listwise ranking problems to the FEMAB ones does not work well for large $l$.
	
	Secondly, we compare PDTR (Algorithm~\ref{AL-TR}) with total ranking algorithms including PLPAC-AMPR \citep{OnlineRankingElicitation2015} and Borda Ranking \citep{MaxingAndRanking2017}. PLPAC-AMPR only works in the pairwise case, so it is not comparable in the listwise case. In the implementations, we vary the $\delta$ value of PLPAC-AMPR and Borda Ranking to balance the trade-off of sample complexity and success rate. The results are illustrated in Figure~\ref{fig:TRSyn}.
	
	\begin{figure}[bht]
		\begin{subfigure}[b]{0.23\textwidth}
			\includegraphics[scale=0.5]{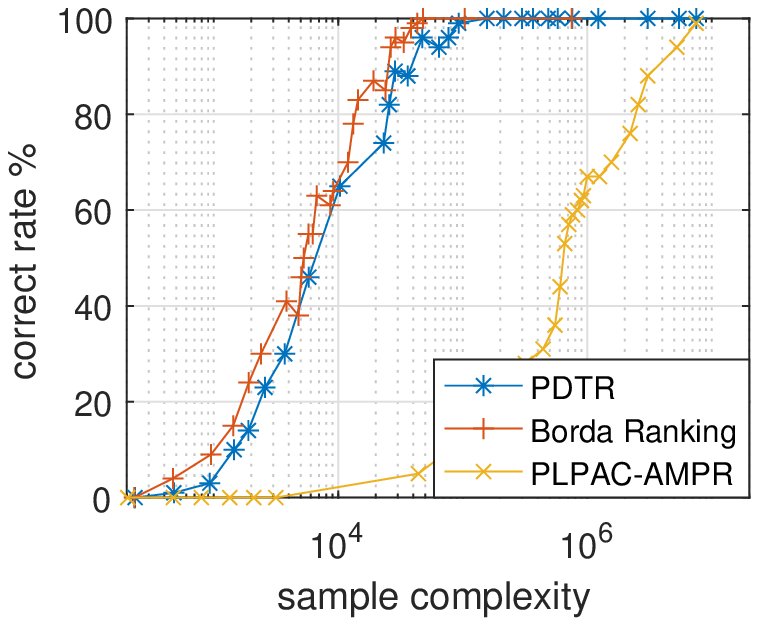}
			\caption{$l=2$.}
		\end{subfigure}
		\begin{subfigure}[b]{0.23\textwidth}
			\includegraphics[scale=0.5]{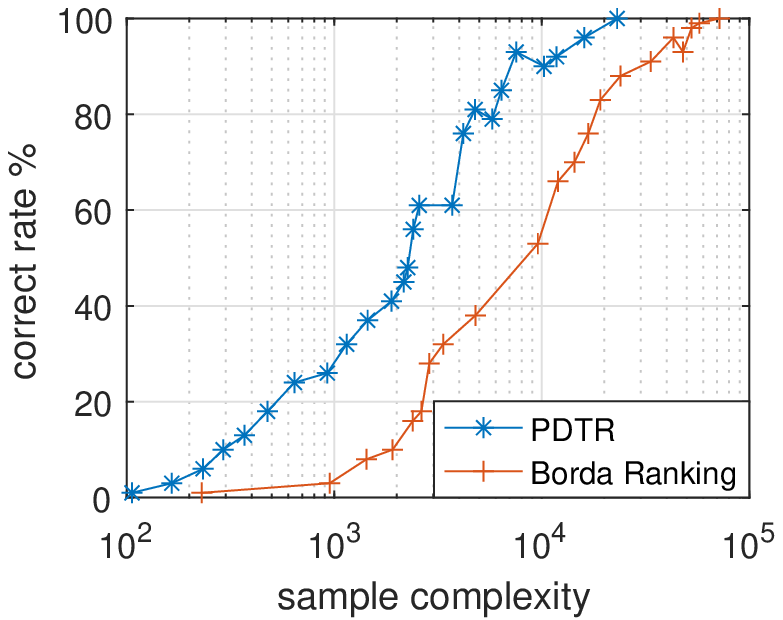}
			\caption{$l=5$.}
		\end{subfigure}
		\caption{Comparisons of total ranking algorithms.}\label{fig:TRSyn}
	\end{figure}
	
	According to the figures, the performance of PLPAC-AMPR is much worse than PDTR, which is consistent with the theoretical results that PLPAC-AMPR's sample complexity is $O(n\log^2{n})$. We can also see that when $l=2$, Borda Ranking is slightly better than PDTR. An explanation is that when $l=2$, Borda Ranking is of sample complexity $O(\frac{n}{\epsilon^2}\log\frac{n}{\delta})$, the same as PDTR, and may have a smaller constant factor. However, when $l=5$, PDTR outperforms Borda Ranking significantly. This again indicates that traditional pairwise ranking algorithms do not work well in the listwise cases. 
	
	Thirdly, we test the performance of TNKS and PDTR under different $l$ values. We show that although their upper bounds of sample complexity increases as $l$ increases, their actual performances are better for larger $l$ values. One possible explanation is that as $l$ increases, the required $\alpha$ values is larger, and the sample complexity upper bound scales as $O(\alpha^{-2})$. The results are illustrated in Figure~\ref{fig:lComparison}. 
	
	\begin{figure}[bht]
		\begin{subfigure}[b]{0.23\textwidth}
			\includegraphics[scale=0.5]{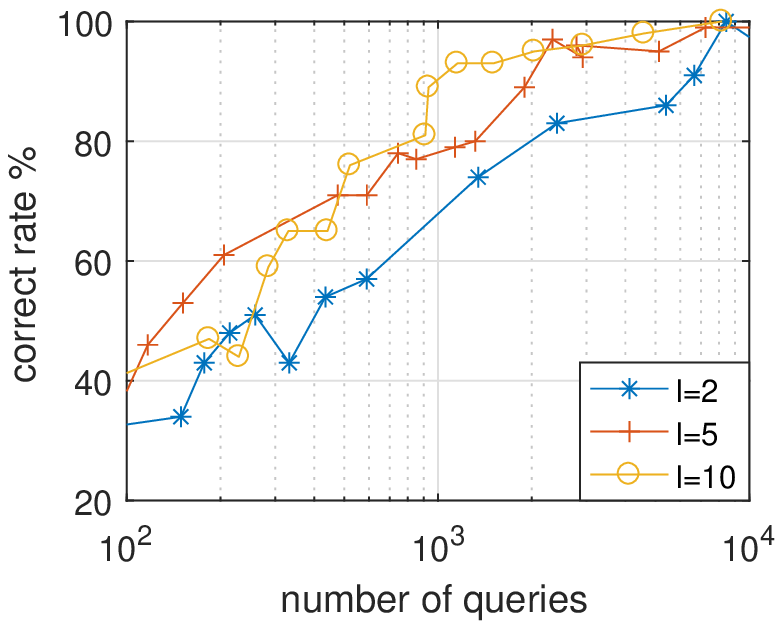}
			\caption{TNKS, $k=2$.}
		\end{subfigure}
		\begin{subfigure}[b]{0.23\textwidth}
			\includegraphics[scale=0.5]{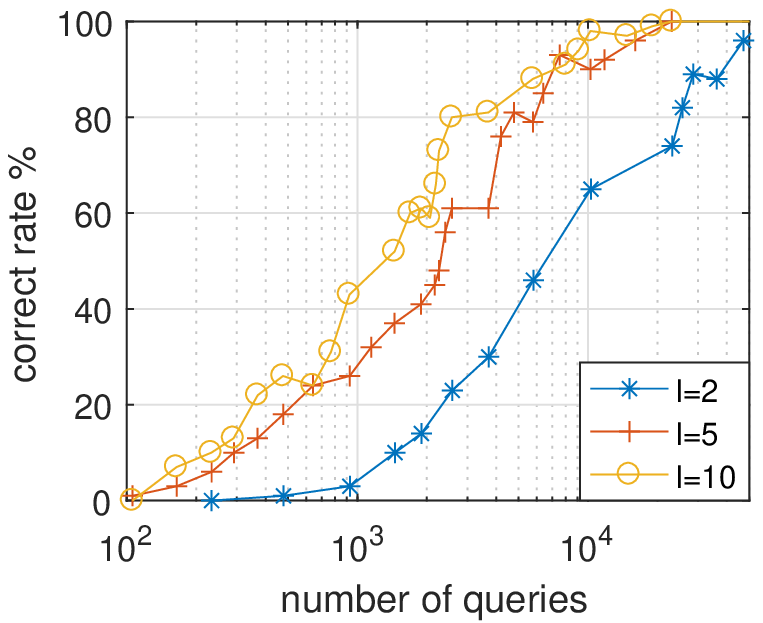}
			\caption{PDTR.}
		\end{subfigure}
		\caption{Performance of the proposed algorithms under different $l$ values.}\label{fig:lComparison}
	\end{figure}
	
	\subsection{Real-World Data}
	In this subsection, we compare the algorithms on real-word data. We use datasets "ED-00004-00000189", "ED-00004-00000190", and "ED-00004-00000198" from PrefLib \citep{PrefLib2007} to conduct real-world data experiments. Each dataset contains several hundreds entries. Each entry provides a strict order of four movies annotated by a user. Here we present four entries of "ED-00004-00000189" to help the readers understand the datasets: "90,1,3,2,4", "45,1,2,3,4", "35,1,3,4,2", and "29,2,3,4,1". The entry "90,1,3,2,4" means that there are 90 users who prefer movie 1 the best, movie 2 the second, and movie 4 the last.
	
	In the implementations of algorithms, we generate the query results by the empirical marginals, that is $\mathbb{P}(i|S)=$ the empirical frequency that $i$ is more preferred than all other items of $S$ in the dataset. We use the corresponding pairwise preference data to compute the preference scores with highest likelihood ratio by MM method \citep{MMMethod2004}, and use them to generate the correct ranking. All algorithms are tested on these three datasets. For each dataset, we perform 100 trials for each point, and then take average over the three datasets. Here, we take parameters $\epsilon=0.05$ and $\delta=0.05$.
	
	\begin{figure}[bht]
		\begin{subfigure}[b]{0.23\textwidth}
			\includegraphics[scale=0.5]{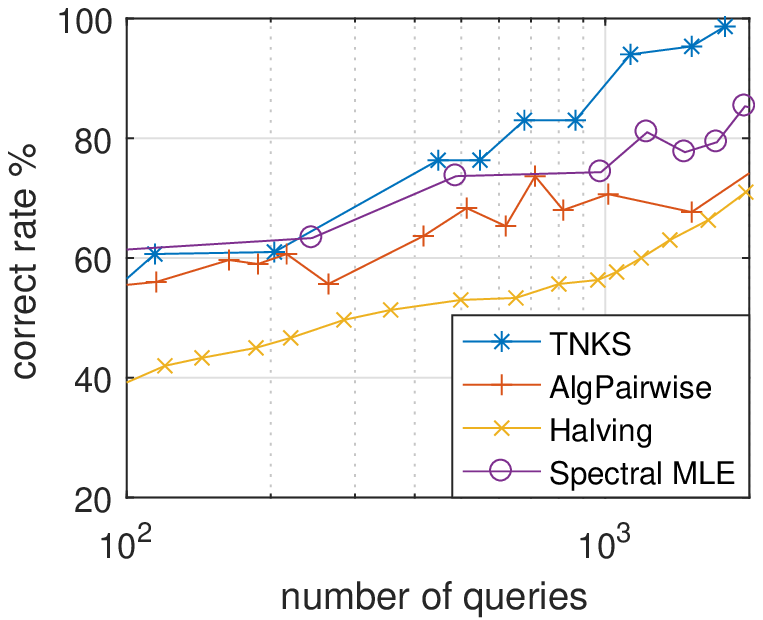}
			\caption{$l=2$.}
		\end{subfigure}
		\begin{subfigure}[b]{0.23\textwidth}
			\includegraphics[scale=0.5]{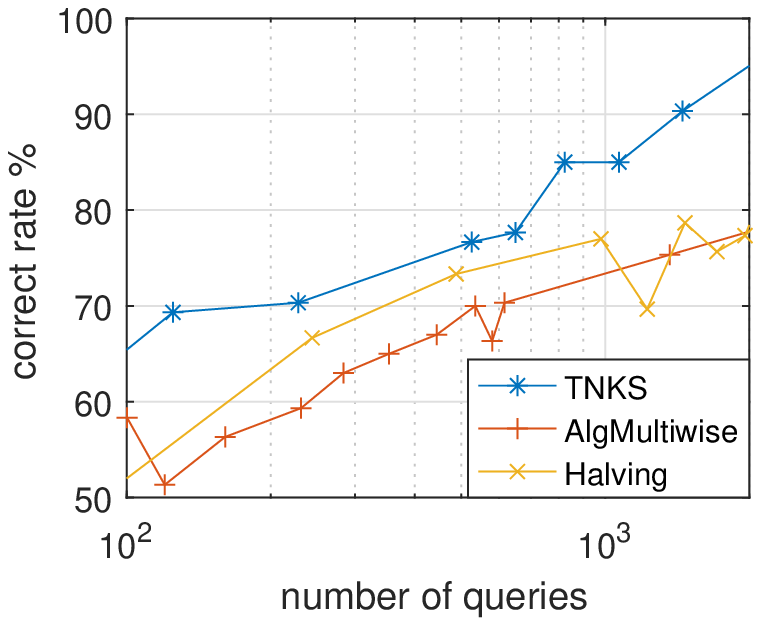}
			\caption{$l=4$.}
		\end{subfigure}
		\caption{Comparisons of the top-$k$ ranking algorithms on real-world data.}\label{fig:TKRW}
	\end{figure}
	
	First we compare the top-$k$ ranking algorithms and the results are shown in Figure~\ref{fig:TKRW}. We can see that TNKS still outperforms other algorithms and the gaps are larger when $l=4$. The results are consistent to our theoretical and numerical results on synthetic data.
	
	Next, we compare the total ranking algorithms, and the results are shown in Figure~\ref{fig:TRRW}. We do not test PLPAC-AMPR, since \citet{OnlineRankingElicitation2015} showed that it does not fit well for some real-word data, especially for those whose empirical marginals are far from the PL model. We ran PLPAC-AMPR on a computer with 8 Intel Core i7-6700 CPUs, but it did not return within a reasonable amount of time. 
	
	\begin{figure}[bht]
		\begin{subfigure}[b]{0.23\textwidth}
			\includegraphics[scale=0.5]{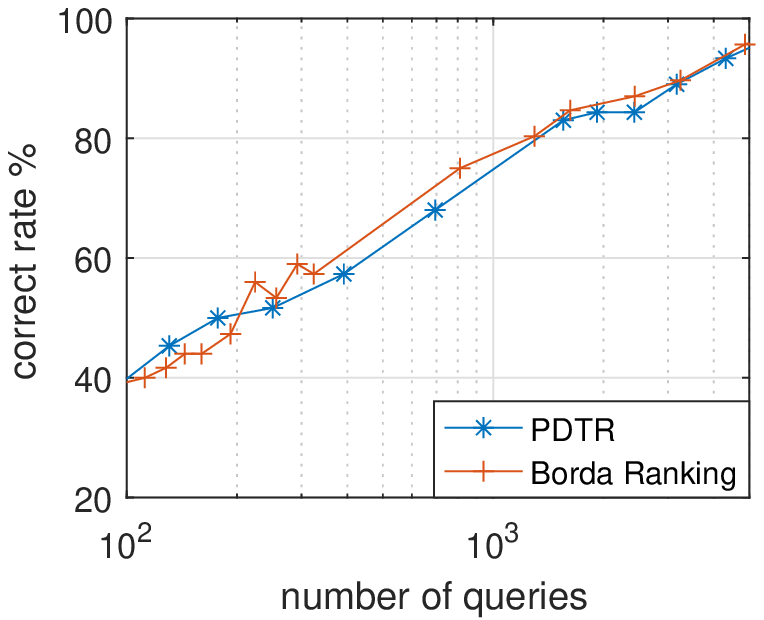}
			\caption{$l=2$.}
		\end{subfigure}
		\begin{subfigure}[b]{0.23\textwidth}
			\includegraphics[scale=0.5]{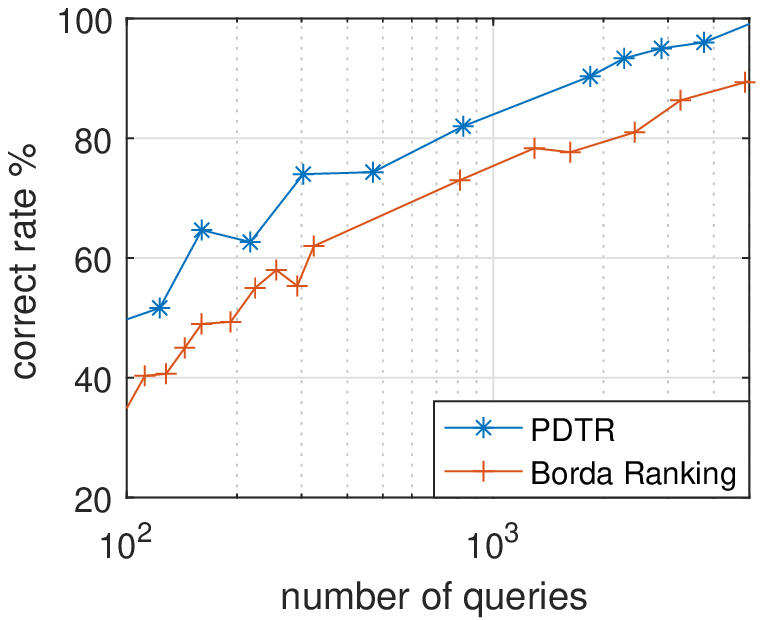}
			\caption{$l=4$.}
		\end{subfigure}
		\caption{Comparisons of the total ranking algorithms on real-world data.}\label{fig:TRRW}
	\end{figure}
	
	According to the results, when $l=2$, the performances of these two algorithms are close. However, when $l=4$, PDTR clearly outperforms Borda Ranking. The results on real-world data are consistent with our theoretical and numerical results on synthetic data.
	
	\section{Conclusion}\label{sec:Conclusion}
	In this paper, we studied the PAC top-$k$ ranking problem and the PAC total ranking problem, both under the MNL model. For the first problem, we derived a lower bound on the sample complexity, and proposed an algorithm that is optimal up to a $\log(k+l)/\log{k}$ factor. When $l=2$ (i.e. pairwise) or $l=O(poly(n))$, our result is order-optimal. For the second problem, we derived a tight lower bound, and proposed an algorithm that matches the lower bound. Numerical experiments on synthetic data and real-world data confirmed the improvement for both pairwise and listwise ranking.
	
\bibliography{MNL_Reference}

\begin{thebibliography}{}

\bibitem[\protect\citeauthoryear{Agarwal \bgroup et al\mbox.\egroup
  }{2017}]{LimitedRounds2017}
Agarwal, A.; Agarwal, S.; Assadi, S.; and Khanna, S.
\newblock 2017.
\newblock Learning with limited rounds of adaptivity: Coin tossing, multi-armed
  bandits, and ranking from pairwise comparisons.
\newblock In {\em In Conference on Learning Theory},  39--75.

\bibitem[\protect\citeauthoryear{Ailon}{2012}]{Ailon2012active}
Ailon, N.
\newblock 2012.
\newblock An active learning algorithm for ranking from pairwise preferences
  with an almost optimal query complexity.
\newblock {\em Journal of Machine Learning Research}.

\bibitem[\protect\citeauthoryear{Baltrunas, Makcinskas, and
  Ricci}{2010}]{RecommendationSystem2010}
Baltrunas, L.; Makcinskas, T.; and Ricci, F.
\newblock 2010.
\newblock Group recommendations with rank aggregation and collaborative
  filtering.
\newblock In {\em Proceedings of the fourth ACM conference on Recommender
  systems},  119--126.
\newblock ACM.

\bibitem[\protect\citeauthoryear{Bennett and Lanning}{2007}]{PrefLib2007}
Bennett, J., and Lanning, S.
\newblock 2007.
\newblock The netflix prize.
\newblock In {\em Proceedings of The KDD Cup and Workshops}.

\bibitem[\protect\citeauthoryear{Bradley and Terry}{1952}]{Bradley1952}
Bradley, R.~A., and Terry, M.~E.
\newblock 1952.
\newblock Rank analysis of incomplete block designs: I. the method of paired
  comparisons.
\newblock {\em Biometrika}  324--345.

\bibitem[\protect\citeauthoryear{Busa-Fekete \bgroup et al\mbox.\egroup
  }{2013}]{Top-kSelection2013}
Busa-Fekete, R.; Szorenyi, B.; Cheng, W.; Weng, P.; and Hüllermeier, E.
\newblock 2013.
\newblock Top-k selection based on adaptive sampling of noisy preferences.
\newblock In {\em International Conference on Machine Learning},  1094--1102.

\bibitem[\protect\citeauthoryear{Cao \bgroup et al\mbox.\egroup
  }{2015}]{Top-kBernoulli2015}
Cao, W.; Li, J.; Tao, Y.; and Li, Z.
\newblock 2015.
\newblock On top-k selection in multi-armed bandits and hidden bipartite
  graphs.
\newblock In {\em In Advances in Neural Information Processing Systems}.

\bibitem[\protect\citeauthoryear{Caragiannis, Procaccia, and
  Shah}{2013}]{SocialChoice2013Caragiannis}
Caragiannis, I.; Procaccia, A.~D.; and Shah, N.
\newblock 2013.
\newblock When do noisy votes reveal the truth?
\newblock In {\em Proceedings of the 14th ACM conference on Electronic
  commerce}.
\newblock ACM.

\bibitem[\protect\citeauthoryear{Chen and Suh}{2015}]{SpectralMLE2015}
Chen, Y., and Suh, C.
\newblock 2015.
\newblock Spectral mle: Top-k rank aggregation from pairwise comparisons.
\newblock In {\em International Conference on Machine Learning},  371--380.

\bibitem[\protect\citeauthoryear{Chen \bgroup et al\mbox.\egroup
  }{2013}]{CrowdSourcing2013}
Chen, X.; Bennett, P.~N.; Collins-Thompson, K.; and Horvitz, E.
\newblock 2013.
\newblock Pairwise ranking aggregation in a crowdsourced setting.
\newblock In {\em In ACM Conference on Web Search and Data Mining},  193–202.
\newblock ACM.

\bibitem[\protect\citeauthoryear{Chen \bgroup et al\mbox.\egroup
  }{2017}]{BothOptimal2017}
Chen, Y.; Fan, J.; Ma, C.; and Wang, K.
\newblock 2017.
\newblock Spectral method and regularized mle are both optimal for top-k
  ranking.
\newblock {\em stat}.

\bibitem[\protect\citeauthoryear{Chen, Gupta, and
  Li}{2016}]{PureExploration2016}
Chen, L.; Gupta, A.; and Li, J.
\newblock 2016.
\newblock Pure exploration of multi-armed bandit under matroid constraints.
\newblock In {\em In International Conference on Machine Learning},  647--669.

\bibitem[\protect\citeauthoryear{Chen, Li, and Mao}{2018}]{Chen2018}
Chen, X.; Li, Y.; and Mao, J.
\newblock 2018.
\newblock A nearly instance optimal algorithm for top-k ranking under the
  multinomial logit model.
\newblock In {\em Proceedings of the Twenty-Ninth Annual ACM-SIAM Symposium on
  Algorithms},  2504--2522.
\newblock SIAM.

\bibitem[\protect\citeauthoryear{Conitzer and
  Sandholm}{2005}]{SocialChoice2005Conitzer}
Conitzer, V., and Sandholm, T.
\newblock 2005.
\newblock Communication complexity of common voting rules.
\newblock In {\em Proceedings of the 6th ACM conference on Electronic
  commerce},  78--87.
\newblock ACM.

\bibitem[\protect\citeauthoryear{Dwork \bgroup et al\mbox.\egroup
  }{2001}]{WebSearch2001}
Dwork, C.; Kumar, R.; Naor, M.; and Sivakumar, D.
\newblock 2001.
\newblock Rank aggregation methods for the web.
\newblock In {\em In Proceedings of the 10th international conference on World
  Wide Web},  613--622.
\newblock ACM.

\bibitem[\protect\citeauthoryear{Falahatgar \bgroup et al\mbox.\egroup
  }{2017a}]{MaxingAndRanking2017}
Falahatgar, M.; Hao, Y.; Orlitsky, A.; Pichapati, V.; and Ravindrakumar, V.
\newblock 2017a.
\newblock Maxing and ranking with few assumptions.
\newblock In {\em In Advances in Neural Information Processing Systems},
  7063--7073.

\bibitem[\protect\citeauthoryear{Falahatgar \bgroup et al\mbox.\egroup
  }{2017b}]{Falahatgar2017}
Falahatgar, M.; Orlitsky, A.; Pichapati, V.; and Suresh, A.~T.
\newblock 2017b.
\newblock Maximum selection and ranking under noisy comparisons.
\newblock In {\em International Conference on Machine Learning},  1088--1096.

\bibitem[\protect\citeauthoryear{Falahatgar \bgroup et al\mbox.\egroup
  }{2018}]{RankingLimits2018}
Falahatgar, M.; Jain, A.; Orlitsky, A.; Pichapati, V.; and Ravindrakumar, V.
\newblock 2018.
\newblock The limits of maxing, ranking, and preference learning.
\newblock In {\em Proceedings of the 35th International Conference on Machine
  Learning}.
\newblock PMLR.

\bibitem[\protect\citeauthoryear{Feige \bgroup et al\mbox.\egroup
  }{1994}]{NoisyComputing1994}
Feige, U.; Raghavan, P.; Peleg, D.; and Upfal, E.
\newblock 1994.
\newblock Computing with noisy information.
\newblock {\em SIAM Journal on Computing} ~5.

\bibitem[\protect\citeauthoryear{Heckel \bgroup et al\mbox.\egroup
  }{2016}]{ActiveRanking2016}
Heckel, R.; Shah, N.~B.; Ramchandran, K.; and Wainwright, M.~J.
\newblock 2016.
\newblock Active ranking from pairwise comparisons and when parametric
  assumptions don't help.
\newblock {\em arXiv preprint}.

\bibitem[\protect\citeauthoryear{Hunter}{2004}]{MMMethod2004}
Hunter, D.~R.
\newblock 2004.
\newblock Mm algorithms for generalized bradley-terry models.
\newblock {\em The annals of statistics}  384--406.

\bibitem[\protect\citeauthoryear{Jang \bgroup et al\mbox.\egroup
  }{2017a}]{ListwisePL2017}
Jang, M.; Kim, S.; Suh, C.; and Oh, S.
\newblock 2017a.
\newblock Optimal sample complexity of m-wise data for top-k ranking.
\newblock In {\em In Advances in Neural Information Processing Systems}.

\bibitem[\protect\citeauthoryear{Jang \bgroup et al\mbox.\egroup
  }{2017b}]{MultiwiseSpectral2017}
Jang, M.; Kim, S.; Suh, C.; and Oh, S.
\newblock 2017b.
\newblock Optimal sample complexity of m-wise data for top-k ranking.
\newblock In {\em Advances in Neural Information Processing Systems}.

\bibitem[\protect\citeauthoryear{Kalyanakrishnan and Stone}{2010}]{Halving2010}
Kalyanakrishnan, S., and Stone, P.
\newblock 2010.
\newblock Efficient selection of multiple bandit arms: Theory and practice.
\newblock In {\em International Conference on Machine Learning}.

\bibitem[\protect\citeauthoryear{Kalyanakrishnan \bgroup et al\mbox.\egroup
  }{2012}]{LowerBound2012}
Kalyanakrishnan, S.; Tewari, A.; Auer, P.; and Stone, P.
\newblock 2012.
\newblock Pac subset selection in stochastic multi-armed bandits.
\newblock In {\em International Conference on Machine Learning},  655--662.

\bibitem[\protect\citeauthoryear{Katariya \bgroup et al\mbox.\egroup
  }{2018}]{CoarseRanking2018}
Katariya, S.; Jain, L.; Sengupta, N.; Evans, J.; and Nowak, R.
\newblock 2018.
\newblock Adaptive sampling for coarse ranking.
\newblock In {\em International Conference on Artificial Intelligence and
  Statistics}.

\bibitem[\protect\citeauthoryear{Lu and Boutilier}{2011}]{SocialChoice2011Lu}
Lu, T., and Boutilier, C.
\newblock 2011.
\newblock Robust approximation and incremental elicitation in voting protocols.
\newblock In {\em IJCAI}.

\bibitem[\protect\citeauthoryear{Luce}{2012}]{BTLModel2012individual}
Luce, R.~D.
\newblock 2012.
\newblock {\em Individual choice behavior: A theoretical analysis}.
\newblock Courier Corporation.

\bibitem[\protect\citeauthoryear{Mohajer and Suh}{2016}]{mohajer2016active}
Mohajer, S., and Suh, C.
\newblock 2016.
\newblock Active top-k ranking from noisy comparisons.
\newblock In {\em Communication, Control, and Computing (Allerton), 2016 54th
  Annual Allerton Conference on}.
\newblock IEEE.

\bibitem[\protect\citeauthoryear{Negahban, Oh, and
  Shah}{2016}]{RankCentrarity2016}
Negahban, S.; Oh, S.; and Shah, D.
\newblock 2016.
\newblock Rank centrality: Ranking from pairwise comparisons.
\newblock In {\em Operations Research},  266--287.

\bibitem[\protect\citeauthoryear{Szörényi \bgroup et al\mbox.\egroup
  }{2015}]{OnlineRankingElicitation2015}
Szörényi, B.; Busa-Fekete, R.; Paul, A.; and Hüllermeier, E.
\newblock 2015.
\newblock Online rank elicitation for plackett-luce: A dueling bandits
  approach.
\newblock In {\em In Advances in Neural Information Processing Systems},
  604--612.

\end{thebibliography}
\bibliographystyle{achemso}

\section{Proof of Theorem~\ref{LB-k-IS}}\label{section9}
	\begin{proof}
		In this proof, we invoke Theorem~8 of \citep{LowerBound2012}, which proves that in the worst case, to find $k$ $(\epsilon,k)$-optimal bandit arms out of $n$ arms, any algorithm needs to take $\Omega(\frac{n}{\epsilon^2}\log\frac{k}{\delta})$ pulls in expectation, where an arm is said to be $(\epsilon,k)$-optimal if its expected reward plus $\epsilon$ is no less than the $k$-th largest expected reward of all the arms. The \textsl{hard instance} used in the proof of Theorem~8 in \citep{LowerBound2012} consists of $n$ arms with Bernoulli rewards. We denote this \textsl{hard instance} as $\mathcal{I}_1$ in this paper.
		
		Suppose that there is an algorithm $\mathcal{A}$ that can solve all instances of the $k$-IS problem with $o(\frac{n}{\epsilon^2}\log\frac{k}{\delta})$ $l$-wise queries. We will show that $\mathcal{A}$ can solve $\mathcal{I}_1$ mentioned above with $o(\frac{n}{\epsilon^2}\log\frac{k}{\delta})$ samples and leads to a contradiction.
		
		In $\mathcal{I}_1$, for any $l$ arms denoted by $a_1,a_2,...,a_l$, let $\mu_i$ be the expected reward of arm $a_i$. For each sample of arm $a_i$, the reward  is an independent instance of the Bernoulli($\mu_i$) distribution. From the proof of Theorem 8 of \cite{LowerBound2012}, there exists a constant $C$ such that $\sup_{i,j}\frac{\mu_i}{\mu_j}\leq C$. Obviously, as $\mu_i\leq 1$ for all $i$, we have $\mu_i\geq 1/C$ for all $i$.
		
		Now define the following procedure $\mathcal{P}$: randomly choose an arm from $\{a_1,a_2,...a_l\}$, and sample it; if the sampling result is 1, then arm $a_i$ is returned; otherwise, we repeat choosing arms until one sample result is 1. In other words, $\mathcal{P}$ continuously chooses arms to sample until an one occurs, and return the corresponding arm.
		
		Now we prove that in $\mathcal{P}$, arm $a_i$ is returned with probability exactly $\frac{\mu_i}{\sum_{i=1}^l{\mu_i}}$ for all $i$, and the procedure $\mathcal{P}$ returns after at most $C$ samples in expectation. 
		
		Define $X$ is the returned value, and $T$ is the number of samples pulled before termination. 
		
		First, it can be proved easily that for $n\in\mathbbm{N}$, 
		\begin{equation} 
		Pr\{T > n\} = \left(1-\frac{1}{l}\sum_{i=1}^l{\mu_i}\right)^n. 
		\end{equation} 
		Since $\mu_i\geq 1/C$ for all $i$, it follows that 
		\begin{equation} 
		\mathbbm{E}T = \sum_{n=0}^\infty\left(1-\frac{1}{l}\sum_{i=1}^l{\mu_i}\right)^n = \frac{l}{\sum_{i=1}^l{\mu_i}} \leq C. 
		\end{equation}
		
		Secondly, we have \begin{align}
		& Pr\{X=a_i\} \nonumber \\ 
		= &\sum_{t=1}^\infty{Pr\{T=t,X=a_i\}}\nonumber \\
		= &\sum_{t=1}^\infty{Pr\{T>t-1\}Pr\{T=t,X=a_i\mid T>t-1\}}\nonumber \\
		=&\sum_{t=1}^\infty{\left[\left(1-\frac{1}{l}\sum_{i=1}^l{\mu_i}\right)^{t-1}\cdot \frac{\mu_i}{l}\right]}\nonumber \\
		= & \frac{\mu_i}{\sum_{i=1}^l{\mu_i}},
		\end{align}
		and the two desired properties of $\mathcal{P}$ have been proved.
		
		The marginal probability of $\mathcal{P}$ is exactly the same as the MNL model. Thus, if we replace the querying operation in $\mathcal{A}$ by $\mathcal{P}$, we can construct a new algorithm that solves the \textsl{hard instance} $\mathcal{I}_1$ by at most $o(\frac{n}{\epsilon^2}\log\frac{k}{\delta})$ samples, and leads to a contradiction.

	\end{proof}

\section{Proof of Lemma~\ref{pi_ij}}\label{section10}
	\begin{proof}
		Fix $\alpha\leq \frac{l-1}{4(l+C-1)}$. Let $S^t$ be the $t$-th set to be queried in the algorithm, and $Q^t$ be its query result. Since one query has only one result, we have that
		\begin{equation} 
			Pr\{Q^t = i \lor Q^t = j\} = Pr\{Q^t = i\} + Pr\{Q^t = j\}.
		\end{equation} 
		Let $p=\frac{1}{2}+\alpha(\theta_i-\theta_j)$. When $S$ does not contain $i$ or $j$ and $S^t=S$, $Q^t$ is with probability $0$ to be $i$ or $j$. Let $S'=S$, and obviously we have 
		\begin{align}\label{pijgs}
			{Pr\{Q^t = i \mid S^t = S\}}\geq & \nonumber\\
			p (Pr\{Q^t = i \mid S^t = S\} + Pr\{Q^t = j & \mid S^t = S'\}).
		\end{align}
		
		For any $l$-sized set $S$ containing $i$ but not $j$, we can find a unique set $S'$ such that $S'=S^t-\{i\}+\{j\}$. Let $\beta = \sum_{a\in S-\{i\}}\theta_a$. Here we note that as $1/C\leq \theta_a \leq 1$ for all items $a$, $\beta\geq \frac{l-1}{C}$. By the definition of the RBC condition, we have 
		\begin{align}
			&\frac{Pr\left\{Q^t = i \mid S^t = S\right\}}{Pr\left\{Q^t = i \mid S^t = S\right\} + Pr\left\{Q^t = j \mid S^t = S'\right\}}\nonumber \\
			= &\frac{\theta_i / (\theta_i + \beta)}{\theta_i / (\theta_i + \beta) + \theta_j / (\theta_j + \beta)} \nonumber\\
			= &\frac{\theta_i(\theta_j + \beta)}{\theta_i(\theta_j + \beta) + \theta_j(\theta_i + \beta)}\nonumber\\
			= &\frac{\frac{\theta_i \theta_j}{\beta} + \theta_i}{\frac{2\theta_i \theta_j}{\beta} + \theta_i + \theta_j} \nonumber\\
			= &\frac{1}{2} + \frac{\frac{1}{2}(\theta_i - \theta_j)}{\frac{2\theta_i \theta_j}{\beta} + \theta_i + \theta_j} \nonumber\\
			\stackrel{(a)}{\geq} &\frac{1}{2} + \frac{\frac{1}{2}(\theta_i - \theta_j)}{ \frac{2}{(l-1)\frac{1}{C}} + 2} \nonumber\\
			= &\frac{1}{2} + \frac{(l-1)(\theta_i-\theta_j)}{4(l-1+C)}\nonumber\\
			\stackrel{(b)}{\geq} &\frac{1}{2} + \alpha (\theta_i-\theta_j) = p,
		\end{align}
		where (a) follows from $\beta\geq \frac{l-1}{C}$, and (b) follows from the condition $\alpha\leq \frac{l-1}{4(l+C-1)}$.
			
		For any $l$-sized set $S$ containing both $i$ and $j$, let $S' = S$. Since $\theta_i,\theta_j\leq 1$, we have
		\begin{align}
		&\frac{Pr\{Q^t = i \mid S^t = S\}}{Pr\{Q^t = i \mid S^t= S\} + Pr\{Q^t = j \mid S^t= S'\}} \nonumber\\
		= & \frac{\theta_i}{\theta_i + \theta_j} = \frac{1}{2} + \frac{\theta_i-\theta_j}{2(\theta_i+\theta_j)} \geq \frac{1}{2}\!+\! \frac{1}{4}(\theta_i-\theta_j)\geq p,
		\end{align} 
		Thus, (\ref{pijgs}) holds when $i,j\in S$.
		
		For any $l$-sized set $S$ containing $j$ but not $i$, there is a unique set $S'=S-\{j\}+\{i\}$ such that (\ref{pijgs}) holds, as $Pr\{Q^t = i \mid S^t = S\}=Pr\{Q^t = j \mid S^t = S'\}=0$.
		
		We also have
		\begin{gather}
			Pr\{Q^t\!=\!i\}\!=\!\sum_{S}{Pr\{S^t\!=\!S\} Pr\{Q^t = i \mid S^t\!=\!S\}},\!\\  
			Pr\{Q^t\!=\!j\}\!=\!\sum_{S'}{Pr\{S^t\!=\!S'\} Pr\{Q^t\!=\!j\!\mid\!S^t\!=\!S'\}}.\!
		\end{gather} 
		For each $S$ and $S'$, the probability that $S^t = S$ is equal to that of $S^t=S'$, and for every possible chosen set $S$, there is a corresponding set $S'$ to let (\ref{pijgs}) hold, and the map from $S$ to $S'$ is one-to-one. Thus, we can conclude the desired result.
	\end{proof}

\section{Proof of Theorem~\ref{TP-TR}}\label{section11}
	\begin{proof}
		\textbf{Step 1.} Fix $\alpha\leq \frac{l-1}{4(l+C-1)}$. Let $\mathcal{E}^t_{i,j}$ be the event that when $i$ and $j$ are in $R$ and $w_i+w_j=t$, $\frac{w_j}{w_i+w_j}\geq b_{w_i+w_j}$. For any items $i$ and $j$ such that $\theta_i>\theta_j+\epsilon$, by Lemma~\ref{pi_ij} and Hoeffding's inequality, it is easy to show that 
		\begin{align}
			Pr\{\mathcal{E}^t_{i,j}\}\leq & \exp\left\{-2t\left(b_{t}-(\frac{1}{2}-\alpha(\theta_i-\theta_j))\right)^2\right\}\nonumber \\ 
			\leq & \exp\left\{-2t\left(b_{t}-\frac{1}{2} +\alpha\epsilon\right)^2\right\}\leq \frac{6\delta^*_1}{\pi^2 t^2}.\label{EijtB} 
		\end{align} 
		Define $\mathcal{E}^{pass}_{i,j}:=\bigcup_{t=1}^\infty\mathcal{E}^t_{i,j}$, i.e., the event that $\frac{w_j}{w_i+w_j}\leq b_{w_i+w_j}$ during the period when $i$ and $j$ are in $R$. If $\mathcal{E}^{pass}_{i,j}$ does not happen, item $j$ will not be marked "\textsl{defeat} $i$" by Line 9 of PDTR. By (\ref{EijtB}), we have that for any items $i$ and $j$ such that $\theta_i>\theta_j+\epsilon$,
		\begin{equation}\label{Epass}
			Pr\{\mathcal{E}^{pass}_{i,j}\}\leq \sum_{t=1}^\infty\frac{6\delta^*_1}{\pi^2 t^2}=\delta^*_1.
		\end{equation} 
		
		Next, we let $\mathcal{E}^{reach}_{i,j}$ be the event that when $i$ and $j$ are in $R$, $w_j$ reaches $\frac{1}{4\alpha^2\epsilon^2}\log\frac{1}{\delta^*_1}$ before $w_i$. For items $i$ and $j$ with $\theta_i>\theta_j+\epsilon$, if $\mathcal{E}^{reach}_{i,j}$ happens, we have $w_j\geq w_i$ when $w_i+w_j=2\lceil\frac{1}{4\alpha^2\epsilon^2}\log\frac{1}{\delta^*_1}\rceil$. By Lemma~\ref{pi_ij} and Hoeffding's inequality, it happens with probability at most
		\begin{align}
			&Pr\left\{B\left(2\lceil\frac{1}{4\alpha^2\epsilon^2}\log\frac{1}{\delta^*_1}\rceil,\frac{1}{2}-\alpha\epsilon\right)\geq \lceil\frac{1}{4\alpha^2\epsilon^2}\log\frac{1}{\delta^*_1}\rceil \right\}\nonumber\\
		 	& \leq \exp\left\{-4\lceil\frac{1}{4\alpha^2\epsilon^2}\log\frac{1}{\delta^*_1}\rceil\alpha^2\epsilon^2\right\} \leq \delta^*_1,
		\end{align}
		where $B(n,p)$ represents a Binomial random variable with parameters $n$ and $p$. Thus, it holds that 
		\begin{equation}\label{Ereach}
			Pr\{\mathcal{E}^{reach}_{i,j}\}\leq \delta^*_1.
		\end{equation} 
		
		\textbf{Step 2: Correctness of PDTR.} Define the set of "bad" events $\mathcal{E}^{bad}:=\bigcup_{i,j\in A:\theta_i>\theta_j+\epsilon}[\mathcal{E}^{pass}_{i,j}\cup \mathcal{E}^{reach}_{i,j}]$. For each $i$ and $j$ with $\theta_i>\theta_j+\epsilon$, $\mathcal{E}^{pass}_{i,j}\cup \mathcal{E}^{reach}_{i,j}$ happens with probability at most $2\delta^*_1$. The number of this kind of $(i,j)$ pairs is at most $\frac{n(n-1)}{2}$. By (\ref{Epass}) and \ref{Ereach}, we have 
		\begin{equation}
			Pr\left\{\mathcal{E}^{bad}\right\} \leq n(n-1)\delta^*_1.
		\end{equation} 
		Now assume that $\mathcal{E}^{bad}$ does not happen. For any items $i,j\in A$ with $\theta_i>\theta_j+\epsilon$, if $j$ ranks higher than $i$ in $\Pi$, then during the time when they are in $R$, either: (i) $w_j$ reaches $\frac{1}{4\alpha^2\epsilon^2}\log\frac{1}{\delta^*_1}$ before $w_i$ or (ii) $j$ \textsl{defeats} $i$ in Line 9, and contradicts the assumption. Thus, the returned value is a correct $\epsilon$-ranking with probability at least $1-n(n-1)\delta^*_1$.
		
		\textbf{Step 3: Sample complexity of PDTR.} In this step, we consider all positive $\alpha$ values. We say a query is \textsl{useful} if its query result (i.e., the winner) is in $R$ at the time when the result is revealed, and is \textsl{useless} otherwise. 
		
		Here we introduce an inequality directly derived from Chernoff Bound. Let $X^1,X^2,...,X^t$ be $t$ independent Bernoulli random variables, and for all $i$, $\mathbbm{E}X^i\geq p$. Define $S:=\sum_{i=1}^t{X^i}$. Let $B(t,p)$ denote a Binomial random variable with parameters $t$ and $p$. For any $b \leq tp$, we have $Pr\{S \leq b\} \leq Pr\{B(t,p)\leq b\}$, and thus, by Chernoff Bound, 
		\begin{equation}\label{ChernoffBound}
			Pr\{S \leq b\} \leq \exp\left\{-\frac{t}{2p}\left(p-\frac{b}{t}\right)^2\right\}.
		\end{equation}
		
		We first prove that with probability at least $1-\delta^*_1$, the number of \textsl{useless} queries is upper bounded by $O(\frac{l}{\alpha^2\epsilon^2}\log{\frac{1}{\delta^*_1}})$. As we can see in Lines 2 and 3 of PDTR, \textsl{useless} queries occur only when $|R|<l$. And after the first time $|R|<l$, every query will involve all the remaining items in $R$. Let $j$ be the last item added to $\Pi$. We have that $j$ wins at most $\frac{1}{4\alpha^2\epsilon^2}\log{\frac{1}{\delta^*_1}}$ queries. It also holds that all \textsl{useless} queries involve item $j$ by Lines 2 and 3. Under the RBC condition, for any query that involves item $j$, the probability that item $j$ wins this query is at least $\frac{1}{lC}$. 
		
		Here we let $b=\frac{1}{4\alpha^2\epsilon^2}\log\frac{1}{\delta^*_1}$ and $m=\lceil 2lC(\frac{1}{\delta^*_1} + b)\rceil$. Say at some time, there have been $m$ queries involving $j$. For each of these queries, item $j$ wins it with probability at least $\frac{1}{lC}$. Let $X$ denote the number of queries $j$ wins. By (\ref{ChernoffBound}), we can get 
		\begin{equation}
			Pr\{X \leq b\} \leq \exp\left\{-\frac{lC}{2}m\left(\frac{1}{lC}-\frac{b}{m}\right)^2\right\} \leq \delta^*_1.
		\end{equation}
		
		Thus, after $m$ queries involving $j$, item $j$ wins at least $\frac{1}{4\alpha^2\epsilon^2}\log\frac{1}{\delta^*_1}$ queries with probability at least $1-\delta^*_1$. Since every \textsl{useless} query involves $j$ and the algorithm terminates immediately after $j$ wins $\frac{1}{4\alpha^2\epsilon^2}\log\frac{1}{\delta^*_1}$ queries, the number of \textsl{useless} queries is upper bounded by $m=O(\frac{l}{\alpha^2\epsilon^2}\log{\frac{1}{\delta^*_1}})$ with probability at least $1-\delta^*_1$. 
		
		As for the number of \text{useful} queries, since any item is removed from $R$ immediately after it wins $\lceil\frac{1}{4\alpha^2\epsilon^2}\log\frac{1}{\delta^*_1}\rceil$ queries, the number of \textsl{useful} queries is at most $O(\frac{n}{\alpha^2\epsilon^2}\log\frac{1}{\delta^*_1})$. Since $\delta^*_1 = \Theta({\delta}/{n})$, the stated sample complexity follows. Combining Step 2 and Step 3 completes the proof.
	\end{proof}

\section{Proof of Lemma~\ref{kappaCorrectness}}\label{section12}
	\begin{proof}
		Fix $\alpha < \frac{l-1}{4(l+C-1)}$. Let $[\kappa]$ be the set of the $\kappa$ most preferred items of $A$. An item is said to be \textsl{bad} if its preference score is less than $\theta'_{[\kappa]}-\epsilon$. Given $i\in[\kappa]$ and \textsl{bad} item $j$, we have $\theta_i > \theta_j+\epsilon$. Following the same steps as in the proof of Theorem~\ref{TP-TR}, we have that during the time when $i$ and $j$ are in $R$, with probability at least $1-2\delta^*_2$, item $j$ does not \textsl{defeat} $i$ in Line 9 and $w_j$ does not reach $\frac{1}{4\alpha^2\epsilon^2}\log\frac{1}{\delta^*_2}$ before $w_i$. The number of this kind of $i,j$ pairs is at most $\kappa(n-1)$. Thus, it follows that with probability at least $1-2\kappa(n-1)\delta^*_2$, the following two statements hold: (i) no item in $[\kappa]$ is \textsl{defeated} by any \textsl{bad} item in Line 9; (ii) no \textsl{bad} item wins $\frac{1}{4\alpha^2\epsilon^2}\log\frac{1}{\delta^*_2}$ queries before any item in $[\kappa]$. 
		
		Now assume that the above two statements hold. We first consider the case where some item $i$ in $[\kappa]$ is not added to $Ans$. To be added to $Ans$, one item must either \text{defeat} $i$ in Line 9 of PDKS or win $\frac{1}{4\alpha^2\epsilon^2}\log\frac{1}{\delta^*_2}$ queries before $i$. When the above two statements hold, no \textsl{bad} item is added to $Ans$, and the desired result holds. For the case where all items in $[\kappa]$ are added to $Ans$, the desired result trivially holds.
		
		The two statements hold with probability at least $1-2\kappa(n-1)\delta^*_2$, and thus, the desired result follows.
	\end{proof}

\section{Proof of Theorem~\ref{TP-TopK1}}\label{section13}
	\begin{proof}
		First fix $\alpha\leq \frac{l-1}{4(l+C-1)}$. Letting $\kappa=k$ in Lemma~\ref{kappaCorrectness}, we have that with probability at least $1-2k(n-1)\delta^*_2$, the returned value is correct. 
		
		As for the sample complexity, we consider all positive $\alpha$ values. The number of \textsl{useful} queries is obviously upper bounded by $O(\frac{n}{\alpha^2\epsilon^2}\log\frac{1}{\delta^*_2}) = O(\frac{n}{\alpha^2\epsilon^2}\log\frac{n}{\delta})$ (recall that $\delta^*_2=\Theta({\delta}/{n})$). 
		
		It remains to prove that the number of \textsl{useless} queries is upper bounded by $O(\frac{l}{\alpha^2\epsilon^2}\log\frac{n}{\delta})$ with probability $1-\delta^*_2$. 
		
		As we can see in Lines 2 and 3 of PDKS, \textsl{useless} queries occur only when $|R|<l$. And after the first time when $|R|<l$, every query will involve all the remaining items in $R$. Let $j$ be the last item added to $Ans$. Item $j$ wins at most $\frac{1}{4\alpha^2\epsilon^2}\log\frac{1}{\delta^*_2}$ queries, and every \textsl{useless} query involves $j$. For each query involving item $j$, by the RBC condition, item $j$ wins this query with probability at least $\frac{1}{lC}$. Let $m = \lceil 2lC( \frac{1}{4\alpha^2\epsilon^2}\log\frac{1}{\delta^*_2}+ \log\frac{1}{\delta^*_2})\rceil$. By (\ref{ChernoffBound}), after $m$ queries involving $j$, the probability that $j$ wins no more than $\frac{1}{4\alpha^2\epsilon^2}\log\frac{1}{\delta^*_2}$ queries is upper bounded by
		\begin{equation}
			\exp\left\{-\frac{lC}{2}m\left(\frac{1}{lC}-\frac{\frac{1}{4\alpha^2\epsilon^2}\log\frac{1}{\delta^*_2}}{m}\right)^2\right\} \leq \delta^*_2.
		\end{equation}
		Thus, the number of \textsl{useless} queries is upper bounded by $m = O(\frac{l}{\alpha^2\epsilon^2}\log\frac{1}{\delta^*_2}) = O(\frac{l}{\alpha^2\epsilon^2}\log\frac{n}{\delta})$ with probability at least $1-\delta^*_2$. This completes the proof.
	\end{proof}

\section{Proof of Theorem~\ref{TP-TopK2}}\label{section14}
	\begin{proof}
		 We prove this theorem by induction. First we fix $\alpha \leq \frac{l-1}{4(l+C-1)}$. Let $\mathcal{E}^{correct}_s$ be the event that $|T_s\cap U_{k,\sum_{r=1}^s\epsilon_r}|\geq k$. $\mathcal{E}^{correct}_0$ happens with probability 1 as $T_0=A$. 
		
		Now we assume that $\mathcal{E}^{correct}_{s-1}$ happens, we want to show $\mathcal{E}^{correct}_{s}$ happens with probability at least $1-\delta_s$. Pick $k$ elements of $T_{s-1}\cap U_{k,\sum_{r=1}^{s-1}\epsilon_r}$, and name them \textsl{key} items. Here we define $\delta':=\frac{\delta_s}{2k(m-1)+1}$. Note that all \textsl{key} items have preference scores at least $\theta_{[k]}-\sum_{r=1}^{s-1}\epsilon_r$. At round $s$, TNKS calls PDKS$(B,k,\delta_r,\epsilon_r, \alpha)$ for $\lceil |T_{s-1}|/m \rceil$ times. Let $\kappa_t$ be the number of \textsl{key} items involved in the $t$-call of PDKS. By Lemma~\ref{kappaCorrectness}, with probability at least $1-2\kappa_t(m-1)\delta'$, the returned value of the $t$-th call contains $\kappa_t$ items with preference scores at least $\theta_{[k]}-\sum_{r=1}^{s}\epsilon_r$. Thus, with probability at least $1-2k(m-1)\delta'$, $\mathcal{E}^{correct}_s$ happens. 
		
		Now, we focus on the sample complexity of round $s$. An item is said to be \textsl{last} if it is the last item added to $Ans$ in some call of PDKS. Each \textsl{last} item wins at most $\frac{1}{4\alpha^2\epsilon^2}\log\frac{1}{\delta'}$ queries before the corresponding call returns. As we can see in Lines 2 and 3 of PDKS, \textsl{useless} queries occur only when $|R|<l$. After the first time $|R|<l$, every query will involve all the remaining items in $R$. Thus, every \textsl{useless} query involves some \textsl{last} item and there are at most $\lceil |T_{s-1}|/m \rceil$ \textsl{last} items in round $s$. Following the similar steps as in the proof of Theorem~\ref{AL-TR}, by (\ref{ChernoffBound}), after at most $2lC(\frac{\lceil |T_{s-1}|/m \rceil}{4\alpha^2\epsilon^2}\log\frac{1}{\delta'} + \log\frac{1}{\delta'}) = O(\frac{|T_{s-1}|}{\alpha^2\epsilon^2}\log\frac{m}{\delta_s})$ queries involving some \textsl{last} item, with probability at least $1-\delta'$, every \textsl{last} items wins $\frac{1}{4\alpha^2\epsilon^2}\log\frac{1}{\delta'}$ queries (implies that all calls of PDKS have returned). Thus, the number of \textsl{useless} queries is upper bounded by $O(\frac{|T_{s-1}|}{\alpha^2\epsilon^2}\log\frac{m}{\delta_s})$ with probability at least $1-\delta'$.
		
		As for the number of \textsl{useful} queries, we have that the number of \textsl{useful} queries of each call of PDKS is upper bounded by $O(\frac{m}{\alpha^2\epsilon^2}\log\frac{m}{\delta_s})$. Noting that $(2k(m-1)+1)\delta' = \delta_s$, thus, for round $s$ of TNKS, with probability at least $1-\delta_s$, it takes $O(\frac{|T_s-1|}{\alpha^2\epsilon^2}\log\frac{m}{\delta})$ queries, and $\mathcal{E}^{correct}_{s}$ happens.
		
		Thus, we can conclude that with probability at least $1-\sum_{r=1}^s\delta_r$, $\mathcal{E}^{correct}_{r}$ happens and round $r$ of TNKS takes at most $O\left(\frac{|T_{r-1}|}{\alpha^2\epsilon_r^2}\log\frac{m}{\delta_r}\right)$ queries for all $r \leq s$.
		
		Noting that $|T_r|=O\left(\frac{n}{2^r}\right)$ (by (\ref{T_rSize})) and $m\leq \min\{2k, l+k-1\}$, it holds that
		\begin{equation}\begin{split}
			\sum_{r=1}^\infty\frac{|T_{r-1}|}{\alpha^2\epsilon_r^2}\log\frac{m}{\delta_r} & = O\left(\sum_{r=1}^\infty\frac{n}{2^r}\left(\frac{5}{4}\right)^r\frac{1}{\alpha^2\epsilon^2}\log{\frac{m r^2}{\delta}}\right) \\
			& = O\left(\frac{n}{\alpha^2\epsilon^2}\log\frac{k+l}{\delta}\right).
		\end{split}\end{equation} 
		We also have $\sum_{r=1}^\infty\delta_r=\delta$. Thus, with probability at least $1-\delta$, the algorithm terminates after $O\left(\frac{n}{\alpha^2\epsilon^2}\log\frac{k+l}{\delta}\right)$ $l$-wise queries, and returns a correct $\epsilon$-top-$k$ subset of $A$.
		
		For $\alpha >  \frac{l-1}{4(l+C-1)}$, since the derivation of sample complexity does not need the condition $\alpha \leq \frac{l-1}{4(l+C-1)}$, with probability at least $1-\delta$, the algorithm returns after $O(\frac{n}{\epsilon^2}\log\frac{k+l}{\delta})$ queries. This completes the proof.
	\end{proof}

\end{document}